\documentclass{article}

\usepackage[preprint]{neurips_2026}
\usepackage{amsmath}
\usepackage{amssymb}
\usepackage{amsthm}

\usepackage[utf8]{inputenc}
\usepackage[T1]{fontenc}
\usepackage{hyperref}
\usepackage{url}
\usepackage{booktabs}
\usepackage{amsfonts}
\usepackage{nicefrac}
\usepackage{microtype}
\usepackage{xcolor}
\usepackage{graphicx}
\graphicspath{{figures/}}
\usepackage{wrapfig}
\usepackage{enumitem}
\usepackage{subcaption}

\usepackage[utf8]{inputenc}
\usepackage{graphicx} 
\usepackage{booktabs} 
\usepackage{makecell} 
\usepackage[dvipsnames]{xcolor} 
\usepackage{amsmath}  

\usepackage{booktabs} 
\usepackage{multirow} 

\newcommand{\up}[1]{\textcolor{ForestGreen}{\scriptsize $\uparrow$#1\%}}
\newcommand{\down}[1]{\textcolor{BrickRed}{\scriptsize $\downarrow$#1\%}}
\newcommand{\valup}[2]{\makecell{#1 \\ \up{#2}}}
\newcommand{\valdown}[2]{\makecell{#1 \\ \down{#2}}}

\newtheorem{theorem}{Theorem}

\newtheorem{lemma}[theorem]{Lemma}

\title{Uncovering and Shaping the Latent Representation of 3D Scene Topology in Vision-Language Models}

\author{%
  Haoming Wang \quad Wei Gao \\
  Department of Electrical and Computer Engineering\\
  University of Pittsburgh\\
  Pittsburgh, PA 15261 \\
  \texttt{\{hw.wang, weigao\}@pitt.edu} \\
}

\begin{document}

\maketitle

\begin{abstract}
Decades of cognitive science establish that humans navigate environments by forming cognitive maps, defined as allocentric and topology-preserving representations of 3D space. While modern Vision-Language Models (VLMs) demonstrate emergent spatial reasoning from 2D egocentric inputs, it remains unclear whether they construct an analogous 3D internal representation. In this paper, we demonstrate that current VLMs do possess a latent topological map of 3D scenes, but it is heavily overshadowed by non-geometric visual semantics, such as color and shape. By isolating this spatial subspace through cross-scene linear feature extraction, we extract a clean spatial subspace that causally controls the model's spatial outputs. We mathematically shape this latent representation and prove its correspondence to the Laplacian eigenmaps of the scene's 3D Gaussian-kernel graph, converging to the physical 3D space in the continuous limit. Motivated by this geometric identification, we further introduce a mathematically principled latent regularization method for VLMs, based on Dirichlet energy. Applying this single-term regularizer to a minimal 500-step supervised VLM fine-tuning (SFT) on simple synthetic data yields significant improvements on real-world spatial benchmarks, outperforming standard SFT and competitive baselines by up to 12.1\% in spatial tasks involving scene topology understanding. Source code is available at https://github.com/pittisl/vlm-latent-shaping



\end{abstract}

\vspace{-0.2in}
\section{Introduction}
\vspace{-0.1in}
\label{sec:intro}

\begin{wrapfigure}{r}{2.8in}
	\centering
	\vspace{-0.25in}
	\includegraphics[width=2.8in]{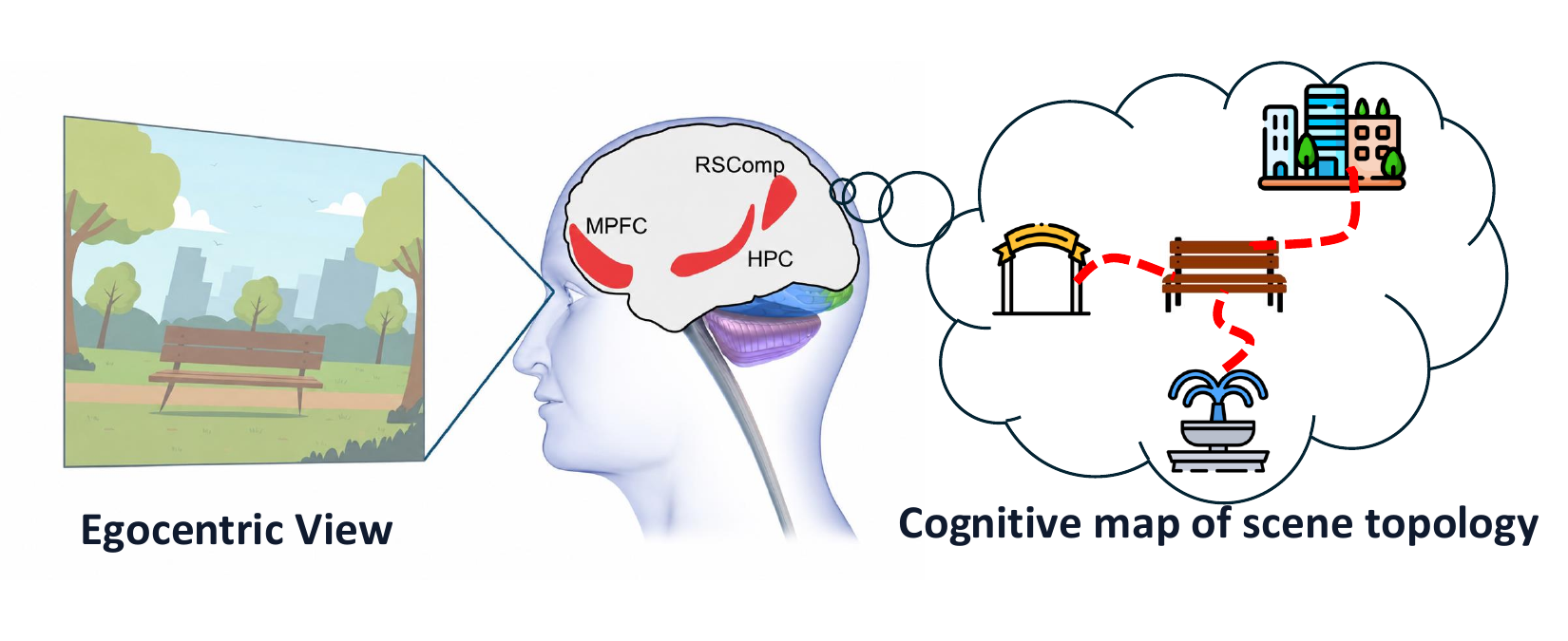}
	\vspace{-0.2in}
	\caption{The human brain compresses egocentric visual inputs into a low-dimensional and topology-preserving cognitive map}
	\vspace{-0.15in}
	\label{fig:cog_map}
\end{wrapfigure}
Decades of cognitive science establish that humans and other animals navigate complex 3D environments by forming an internal \emph{cognitive map} --- an allocentric representation of the spatial relations among objects that can be queried without visually perceiving the scene~\cite{tolman1948cognitive,o1978hippocampus}. Rather than individual pixels, image patches or metric distances, what such a cognitive map preserves is primarily the \emph{scene topology}, defined as the qualitative graph of ``what is next to what'' that encodes relative positions and neighborhood structures~\cite{warren2019non,peer2021structuring}, as shown in Figure \ref{fig:cog_map}. By constructing the cognitive map, the human brain compresses complex visual inputs into a low-dimensional representation, hence supporting efficient but precise spatial understanding and reasoning. 

Nowadays, Vision-Language Models (VLMs) \cite{li2024llava,maaz2024video} can succeed on simple spatial reasoning tasks given only egocentric visual inputs instead of dedicated 3D geometric information (e.g., depth estimation) ~\cite{rajabi2024gsr,ma2022sqa3d,liao2024reasoning}, but they fail on more difficult tasks in complex scenes with many objects, partial views and cluttered layouts \cite{liu2025spatial,wang2025infinibench}. The natural question, then, is whether VLMs have topology-preserving latent representations of the scene's 3D structure that are analogous to the cognitive map in human brain. It is important to know \emph{where} such representations reside, and whether we can mathematically or experimentally \emph{shape} them to improve the model's spatial reasoning abilities.

Existing diagnostic tools fall short of answering these questions. 
Attention-map analyses on VLMs reveal which 2D image patches the answer token attends to, but cannot reveal whether the model internally binds those image patches to a 3D allocentric representation~\cite{chefer2021generic,chen2025spatial,qi2025beyond}. Mechanistic interpretability for multimodal models has identified linear circuits for relational reasoning~\cite{park2024geometry,kang2026linear,assouel2025visual}, but these analyses remain anchored to the 2D image grid 
or text features, and stop short of asking whether the residual stream encodes the scene's intrinsic 3D topology. End-to-end VQA accuracy, in particular, is itself an unreliable proxy, as VLMs are known to exploit caption-style co-occurrence priors, distractor-set artifacts, and language shortcuts that may result in good accuracy on certain tasks but lack generalizability \cite{agrawal2016analyzing,goyal2017making,shi2026vision}.
Instead, we need a direct probe of \emph{latent} representations in model's activations, which (i) does not rely on the model's text output, (ii) is benchmark-agnostic, and (iii)establishes a precise mathematical equivalence to the physical 3D scene layout.

\begin{figure}
	\centering
	\vspace{-0.15in}
	\includegraphics[width=0.9\linewidth]{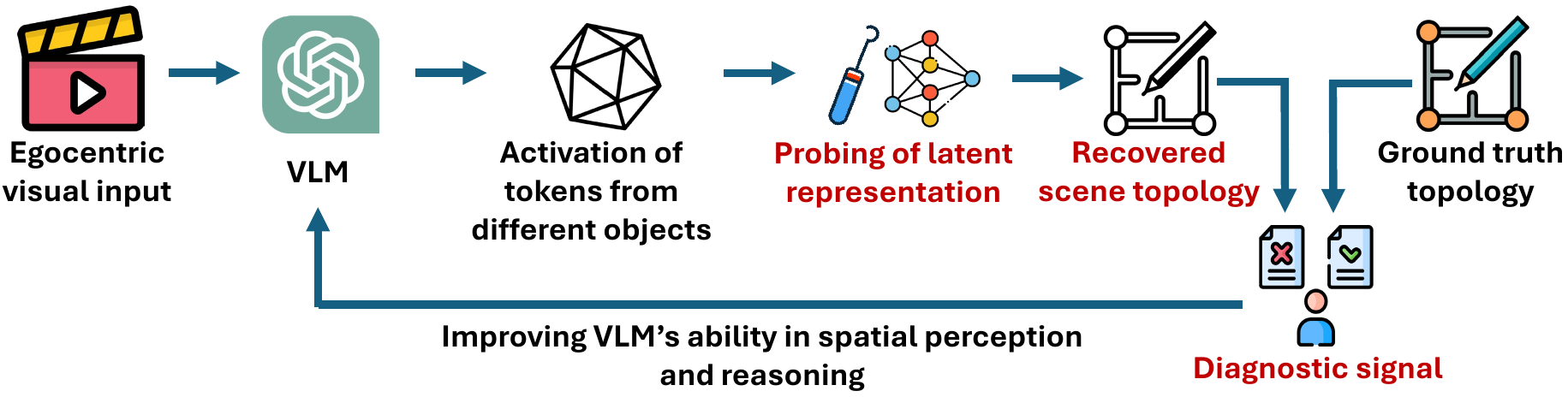}
	\vspace{-0.05in}
	\caption{Uncovering and shaping the latent representation of 3D scene topology using a feature extraction probe as diagnostic signal}
	\vspace{-0.2in}
	\label{fig:overview}
\end{figure}

Our findings showed that the existing VLMs do encode information about the 3D scene topology, but this representation is severely overshadowed by spatially irrelevant visual semantics (\S\ref{sec:diagnose}). To bypass this semantic entanglement, as shown in Figure \ref{fig:overview}, this paper introduces a linear feature extraction technique that isolates visual semantics from latent representations of 3D scene topology.
More specifically, we consider the same object's non-spatial attributes (e.g., color, shape, etc) as constant across different scenes, and average the object-token activations to project out the irrelevant visual semantics and extract a dominant spatial identity subspace.
Experiment results validated that principal components of the extracted spatial subspace cleanly recover the physical 3D scene layout (\S\ref{sec:extract}). 

Furthermore, we analytically proved that this extracted spatial subspace optimally aligns with the Laplacian eigenmaps of the scene's 3D graph (Theorem~\ref{thm:eigenmap}) and converges to 3D coordinate space (Theorem~\ref{thm:limit}). Consequently, we introduce a Dirichlet-energy regularizer to fine-tune the VLM's ``spatially relevant'' layers. This regularizer can be theoretically ensured to guarantee the task reliability, drastically reduce the sample complexity and minimize risk (Theorems~\ref{thm:realizability}--\ref{thm:risk}). Our experiment results show that just 500 supervised fine-tuning (SFT) steps with LoRA on simple synthetic data could successfully reshape the internal geometry in VLM's latent representations, and yield significant improvements on real-world spatial benchmarks such as VSI-Bench \cite{yang2025thinking} and MindCube \cite{wang2025mindcube}, without architectural changes or introducing extra 3D modalities to the model itself (\S\ref{sec:enforce}).  In spatial tasks involving scene topology understanding, such improvements could be up to 12.1\%, compared to standard SFT and competitive baselines of spatial encoding.

\vspace{-0.1in}
\section{Related Work}
\vspace{-0.1in}
\label{sec:related_works}

\subsection{Spatial Reasoning in Vision-Language Models}
\vspace{-0.05in}
While cognitive science highlights that biological agents navigate using topology-preserving ``cognitive maps''~\cite{tolman1948cognitive, o1978hippocampus, peer2021structuring, warren2019non}, imparting such allocentric 3D understanding to VLMs has largely been approached as an explicit architectural or brute-force data scaling problem. Recent approaches attempt to force 3D awareness by appending depth modalities~\cite{wang2025mindcube,wang2026mosaicthinker} or fusing multi-vision encoders~\cite{li2024llava, maaz2024video,cheng2024spatialrgpt}. Other efforts scale up spatial capabilities by pretraining on large amounts (e.g., billions) of synthesized spatial QA samples~\cite{rajabi2024gsr, liao2024reasoning, liu2025spatial,cai2025scaling}, or utilizing reinforcement learning (RL) with programmatic spatial rewards~\cite{batra2025spatialthinker,liu2025spatial,wu2025reinforcing}. While highly effective, these approaches rely heavily on injecting metric 3D data or altering standard input pipelines. In contrast, guided by the topological nature of biological spatial representation, our work explores whether standard VLMs naturally form  latent 3D representations~\cite{qi2025beyond, park2024geometry} solely from 2D egocentric visual inputs, avoiding the need for dedicated external 3D modules.

\vspace{-0.1in}
\subsection{Mechanistic Interpretability of Geometry in VLMs}
\vspace{-0.05in}
Mechanistic interpretability seeks to reverse-engineer the latent algorithms learned by neural networks~\cite{geiger2021causal, vig2020investigating,nanda2023emergent}. In multimodal models, researchers have utilized attention-map analyses to explore how VLMs ground textual concepts in image patches~\cite{chefer2021generic, agrawal2016analyzing, goyal2017making, chen2025spatial}, and recent works strived to directly map between spatial and relational representations~\cite{assouel2025visual,huang2025decomposing,nanda2023emergent}. ~\cite{li2025does} demonstrated that ViT patch embeddings encode an \emph{is-same-object} quadratic form to bind visual features, and Kang et al.~\cite{kang2026linear} identified linear ``spatial-ID'' vectors that VLMs use to bind spatial locations to text tokens, showing that causal interventions on these IDs can flip a model's spatial output. Despite these advances, existing analyses remain largely anchored to the 2D image grid (i.e., correspondence between 2D image patches and objects in the scene) or isolated within the vision encoder's early binding mechanisms~\cite{shi2026vision}. This leaves an open question about whether the downstream language modeling residual stream constructs a cohesive 3D allocentric understanding of the physical scene.

\vspace{-0.1in}
\section{VLMs' Latent Representations of 3D Scene Topology}
\vspace{-0.1in}
\label{sec:diagnose}

In this section, we conduct preliminary experiments to determine \emph{whether} pretrained VLMs already represent 3D scene topology in their latent space and \emph{how strong} it is relative to the irrelevant semantics representations that the model also carries. 

\begin{figure}[h]
	\centering
	\vspace{-0.15in}
	\includegraphics[width=0.7\linewidth]{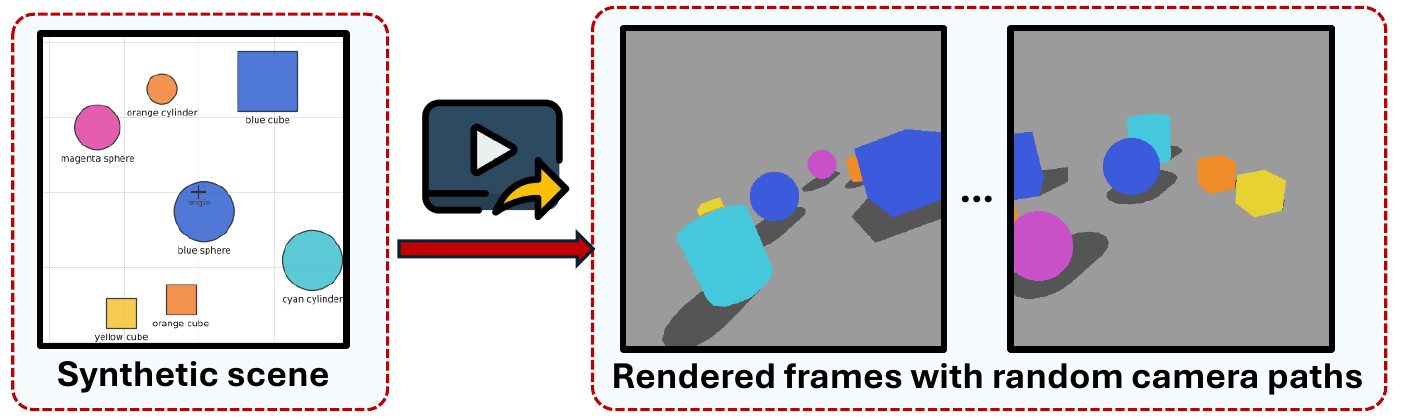}
	\caption{Synthetic spatial data}
	\vspace{-0.2in}
	\label{fig:data_example}
\end{figure}

\vspace{-0.1in}
\paragraph{Experiment Setup:}
To rigorously evaluate the internal representations of VLMs without the confounding variables of real-world images (e.g., lighting, textures, background context, occlusion, etc) \cite{xu2025spatialbench,wang2025spatialvid}, we built and used a fully controllable synthetic spatial dataset, namely \textbf{SynSpat3D}, to decouple a model's purely geometric understanding from its semantic visual priors.
As shown in Figure \ref{fig:data_example}, each synthetic sample is generated by a two-step pipeline: (1) \emph{Scene generation}: 6 to 12 objects are randomly placed in a bounded 3D coordinate space, with simple shapes (e.g., cube, sphere) and colors (e.g., red, blue) acting as the only independent sources of variance. (2) \emph{Frame generation}: We render the scene using randomly generated camera paths to provide egocentric 2D views. This dataset contains 1,000 distinct scenes, each of which is paired with 5 programmatically generated spatial questions, yielding a total of 5,000 video-question pairs. More details can be found in Appendix \ref{app:data_gen}. 

Using this dataset, we examine the residual-stream activations at every decoder layer of two pretrained VLMs, namely Qwen2.5-VL-7B \cite{qwen2.5-VL} and InternVL3-8B \cite{zhu2025internvl3}, which are widely adopted open-sourced VLMs with strong baseline spatial reasoning capabilities.

\begin{wrapfigure}{r}{2.4in}
	\centering
	\vspace{-0.15in}
	\includegraphics[width=\linewidth]{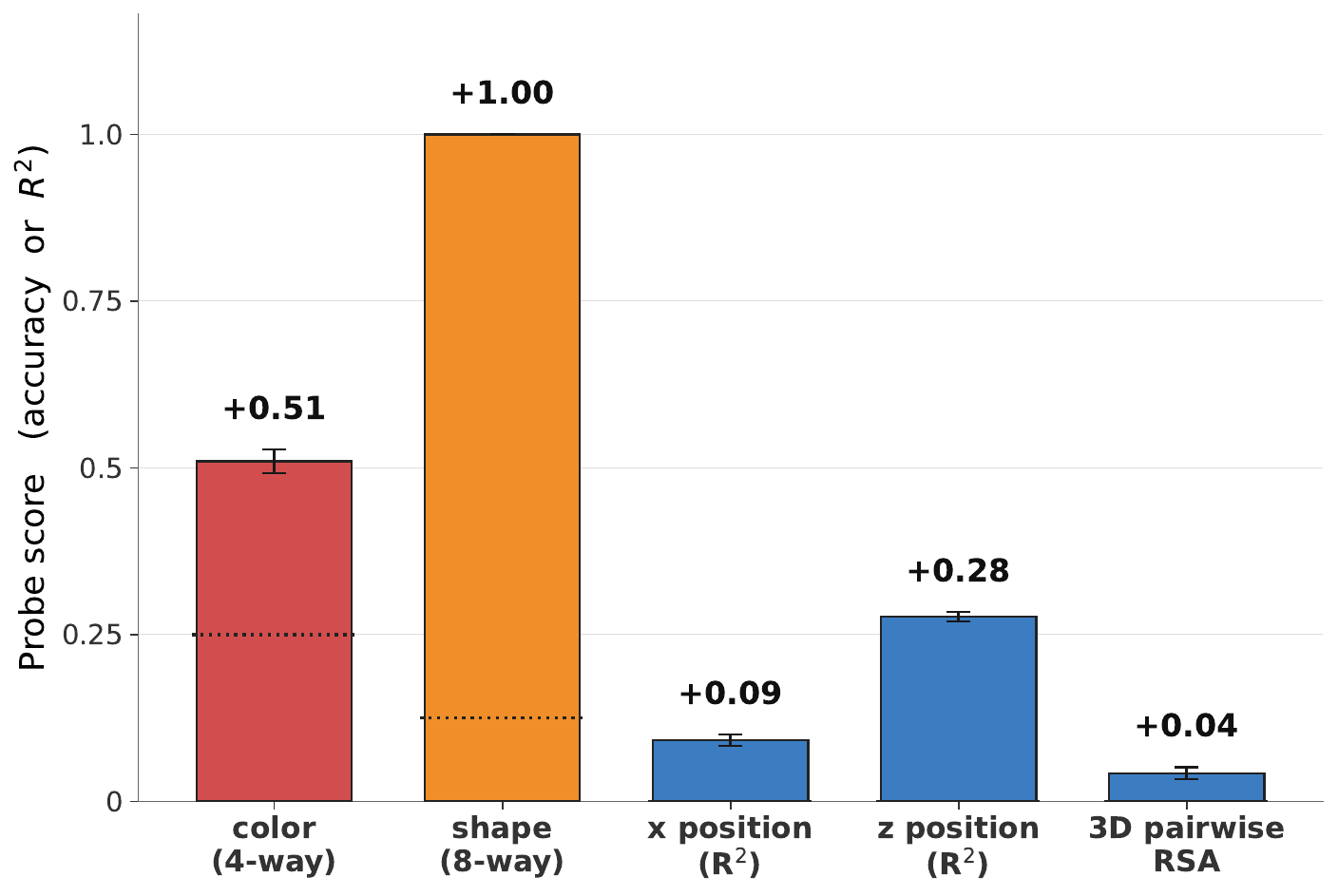}
	\vspace{-0.2in}
	\caption{VLMs preferentially encode irrelevant visual features over 3D position}
	\vspace{-0.1in}
	\label{fig:attention_dominance}
\end{wrapfigure}


\vspace{-0.1in}
\paragraph{Dominance of spatially irrelevant features.} 
We first apply direct linear probing \cite{vig2020investigating} at different layers to quantify the representational richness of spatially relevant features (3D positions) versus spatially irrelevant semantics (color, shape, etc). To do so, we train 
probes on the activations of image patches to predict each of these targets: object color, object shape, physical $x$-coordinate, physical $z$-coordinate, and pairwise 3D distance using representational similarity analysis (RSA) that measures the relative distance between objects. As shown in Figure ~\ref{fig:attention_dominance}, the results reveal a stark representational imbalance. The visual identity of an object is recoverable from activations with near-perfect accuracy, but its true 3D position can only be partially recovered. Analyzing the activation variance captured by each feature's optimal low-rank subspace further highlights this discrepancy (details are in Appendix~\ref{app:diagnose}).




\vspace{-0.1in}
\paragraph{VLM outputs in spatial reasoning may follow semantics instead of 3D positions:} 
To determine whether this representational imbalance dictates the VLM's downstream reasoning, we conduct a behavioral counterfactual experiment \cite{geiger2021causal}. We evaluate the model across 50 scenes under three conditions: the \emph{original} scene, a \emph{color-swap} variant (colors permuted, 3D positions preserved), and a \emph{position-swap} variant (colors preserved, 3D layout scrambled). 

We found that simply shuffling the object colors can flip the model's spatial predictions in 41.0\% of cases on distance-order questions. This semantic disruption is nearly as severe as physically scrambling the 3D layout, which yields a 48.6\% flip rate. This extreme sensitivity provides direct behavioral evidence that the VLM fails to cleanly separate ``where'' an object is from ``what'' it looks like. Instead of querying a robust geometric map, the language modeling head relies heavily on dominant color representations, rendering its spatial reasoning brittle to superficial semantic changes. More details about this can be found in Appendix~\ref{app:diagnose}.

\vspace{-0.1in}
\section{Linear Spatial Feature Extraction via Cross-Scene Averaging}
\vspace{-0.1in}
\label{sec:extract}

To address such representational imbalance, we aim to cleanly disentangle the latent spatial representation from the dominant non-geometric representations, and our approach is a linear feature extraction technique based on cross-scene averaging that projects out the irrelevant visual semantics. Since a VLM's latent space is fundamentally shaped by its pretraining objectives to form consistent, high-dimensional geometric structures, the extracted spatial subspace corresponds to how a VLM inherently maps 2D visual tokens into a cohesive allocentric 3D understanding, and is hence dependent only on the VLM's internal weights and architecture rather than the  dataset being used in extraction. 

As a result, we use the SynSpat3D dataset built in Section \ref{sec:diagnose} for extraction, and validate the efficacy of the extracted spatial subspace through both qualitative geometric visualizations and quantitative formal analysis. In Section \ref{sec:enforce}, we will further mathematically formalize this subspace to derive a Dirichlet-energy regularizer, demonstrating that enforcing this topology-preserving penalty in VLM fine-tuning can significantly improve the VLM's spatial reasoning abilities in real-world scenes.

\vspace{-0.1in}
\subsection{Per-object Representation}
\vspace{-0.1in}
\label{subsec:object_repr}

Let $\mathcal{S}$ denote the set of scenes and let object identity $o \in \mathcal{O}$ index a colored geometric primitive (e.g.\ ``red cube''). Each scene\footnote{Each scene contains at most one instance of a given identity.} $s\in\mathcal{S}$ contains $m$ objects $\{o_i\}_{i=1}^m$ at 3D positions $X^{(s)} = (x_1^{(s)},\dots,x_m^{(s)}) \in \mathbb{R}^{3 \times m}$. Unlike texts where discrete concepts map cleanly to individual tokens, visual objects in VLMs are distributed across continuous spatio-temporal patch grids, requiring a deliberate aggregation strategy to isolate their representations.

\begin{figure}
	\centering
	\vspace{-0.25in}
	\includegraphics[width=0.7\linewidth]{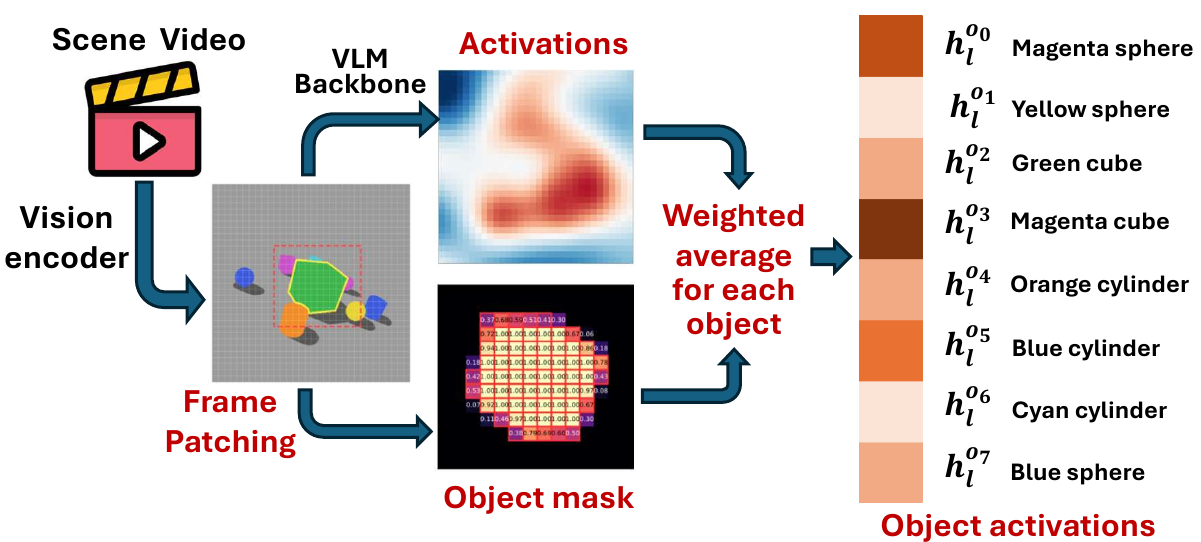}
	\caption{Computation of per-object representation}
	\vspace{-0.2in}
	\label{fig:per_obecjt_ac}
\end{figure}



As shown in Figure \ref{fig:per_obecjt_ac}, we feed the VLM with a sequence of rendered frames $\{I_t^{(s)}\}_{t=1}^{T}$. Each frame is patchified by the vision encoder onto a $g_h\!\times\!g_w$ patch grid, and consecutive frames are fused so that $T$ input frames produce $T/\tau$ temporal tokens\footnote{$\tau$ is used for VLM backbones with a temporal-patch merger of size $\tau$. For example, $\tau{=}2$ in Qwen2.5-VL.}. Let $h^{(\ell)}_{p,t}\in\mathbb{R}^d$ denote the activation at decoder layer $\ell$ of the visual token at patch indexed by $p \in \mathcal{P} := \{1,\dots,g_h\}\times\{1,\dots,g_w\}$ in temporal slot $t$. For each object $o_i$ in scene $s$ we form a mask--driven, coverage-weighted pool over visual tokens:
\begin{equation}
	h^{(s,o_i)}_{\ell,t} \;=\;
	\frac{\sum_{p:\,c^{(s)}_{o_i,p,t}\ge\kappa}\,c^{(s)}_{o_i,p,t}\;h^{(\ell)}_{p,t}}
	{\sum_{p:\,c^{(s)}_{o_i,p,t}\ge\kappa}\,c^{(s)}_{o_i,p,t}},
	\quad
	c^{(s)}_{o_i,p,t} \;=\;
	\frac{\bigl|\,\text{patch}_p\,\cap\,\text{mask}_{o_i,t}^{(s)}\,\bigr|}
	{\bigl|\,\text{patch}_p\,\bigr|},
	\label{eq:pool}
\end{equation}
where $c^{(s)}_{o_i,p,t}\in[0,1]$ is the fraction of patch $p$'s pixels belonging to object $o_i$ in the ground-truth mask (temporal-slot averaged), and $\kappa$ is a coverage threshold that excludes patches grazed by an object's silhouette. We then average across temporal slots to obtain the
\emph{per-object activation}:
\vspace{-0.05in}
\begin{equation}
	h^{(s,o_i)}_\ell \;=\;
	\frac{1}{T/\tau}\sum\nolimits_{t=1}^{T/\tau} h^{(s,o_i)}_{\ell,t}
	\;\in\; \mathbb{R}^{d},
	\qquad
	H^{(s)}_\ell \;=\;
	\bigl[h^{(s,o_1)}_\ell,\,\dots,\,h^{(s,o_m)}_\ell\bigr]^{\!\top}
	\;\in\; \mathbb{R}^{m \times d},
	\label{eq:obj_repr}
	\vspace{-0.05in}
\end{equation}
where $d$ is the LM hidden size 
. $H^{(s)}_\ell$ is the extracted activation matrix for objects in different layers.

\subsection{Disentangling the Spatial Representation}
\vspace{-0.05in}
\label{subsec:identity_repr}

\paragraph{Observation of linear decomposition:}
Recent works of mechanistic interpretability \cite{tigges2023linear,trager2023linear,wang2025deciphering,park2023linear,jiang2024origins} observed that, in mainstream model architectures and both text and vision domains, high-level human-interpretable features are encoded as approximately linear directions in the latent space \cite{kang2026linear,huang2025decomposing,rajaram2025line,bhalla2024interpreting,tehenan2025linear}. This means that for an object's representation, it can be decomposed additively into an identity-attribute component (carrying properties that depend only on $o$ such as color, shape, lexical identity, etc) and a spatial component (carrying the object's 3D position in scene $s$):
\vspace{-0.05in}
\begin{equation}
	h^{(s,o)}_\ell \;=\; u_{\rm id}(o) \;+\; u_{\rm sp}\!\bigl(x^{(s)}_o\bigr) \;+\; \varepsilon^{(s,o)},
	\label{eq:lrh}
\end{equation}
where $u_{\rm id}\colon \mathcal{O}\to\mathbb{R}^d$ is identity-conditional, $u_{\rm sp}\colon \mathbb{R}^3\to\mathbb{R}^d$ is position-conditional, and $\varepsilon^{(s,o)}$ collects interactions and noise. 
Hence, to isolate the spatial component $u_{\rm sp}$ , we first compute the identity-conditional component $u_{\rm id}$, by averaging object activations across randomized scenes.

\paragraph{Cross-scene averaging:}
Because the 3D layouts in our synthetic dataset are generated uniformly at random, the spatial component $u_{\rm sp}$ for a specific object identity has an expected value of zero \cite{kang2026linear}. Consequently, averaging an object's activations across a large corpus of scenes effectively cancels out its position-dependent variance, leaving only the invariant identity features. Thus, we have 
\begin{equation}
	\bar h^{(o)}_\ell \;\;\triangleq\;\; \frac{1}{|\mathcal{S}_o|}\!\sum_{s\in\mathcal{S}_o}\! h^{(s,o)}_\ell
	\;\;=\;\; u_{\rm id}(o) \;+\; \underbrace{\frac{1}{|\mathcal{S}_o|}\!\sum_{s\in\mathcal{S}_o} u_{\rm sp}\!\bigl(x^{(s)}_o\bigr)}_{\to\,0\;\text{as}\;|\mathcal{S}_o|\to\infty}
	\;+\; O(|\mathcal{S}_o|^{-1/2}),
	\label{eq:cross_scene_mean}
\end{equation}
so the cross-scene mean activation $\bar h^{(o)}_\ell$ is a consistent estimator of the identity-attribute prototype $u_{\rm id}(o)$. 
We collect the per-identity prototypes into a matrix
$\bar H_\ell \;=\;[\bar h^{(o_1)}_\ell,\,\dots,\,\bar h^{(o_K)}_\ell]^\top \in \mathbb{R}^{K\times d}$
over the $K$ unique identities in our object inventory.

\vspace{-0.1in}
\paragraph{Identity-attribute basis $W$:}
Having estimated the identity prototypes, we next construct an orthogonal basis for this spatial-irrelevant subspace so that it can be cleanly projected out. A thin SVD $\bar H_\ell^\top = U \Sigma V^\top$ gives the principal directions of the prototype cloud; we keep the top-$k$ left singular vectors as columns of an orthonormal basis:
\begin{equation}
	W_\ell \;\in\; \mathbb{R}^{d \times k},\quad W_\ell^\top W_\ell = I_k,
	\qquad
	\mathrm{span}(W_\ell) \;\approx\; \mathrm{span}\{u_{\rm id}(o) : o\in\mathcal{O}\},
	\label{eq:basis}
\end{equation}
and the specific choice of $k$ is discussed in Appendix~\ref{app:fulltable}. 

\vspace{-0.1in}
\paragraph{Linear spatial feature extraction $\tilde h^{(s,o)}_\ell$:}
\label{subsec:resid}

The orthogonal projector \emph{onto} the identity-attribute subspace is $P_W = W_\ell W_\ell^\top$, and the projector
\emph{away from} it is 	$P_\perp \;\triangleq\; I_d \;-\; W_\ell W_\ell^\top$. We then define the \emph{spatially-extracted} (or
\emph{spatial-feature} for short) per-object activation as
\begin{equation}
	\tilde h^{(s,o)}_\ell \;\triangleq\; P_\perp \, h^{(s,o)}_\ell
	\;=\; h^{(s,o)}_\ell \;-\; W_\ell W_\ell^\top h^{(s,o)}_\ell,
	\label{eq:residualized}
\end{equation}
which under Eq.~(\ref{eq:lrh}) approximately equals $u_{\rm sp}(x^{(s)}_o) + \varepsilon^{(s,o)}$ --- the spatial-relevant component plus residual noise. The extracted per-scene object matrix is then $\tilde H^{(s)}_\ell = H^{(s)}_\ell P_\perp$. 


\begin{figure}[ht]
	\centering
	\vspace{-0.05in}
	\includegraphics[width=1\linewidth]{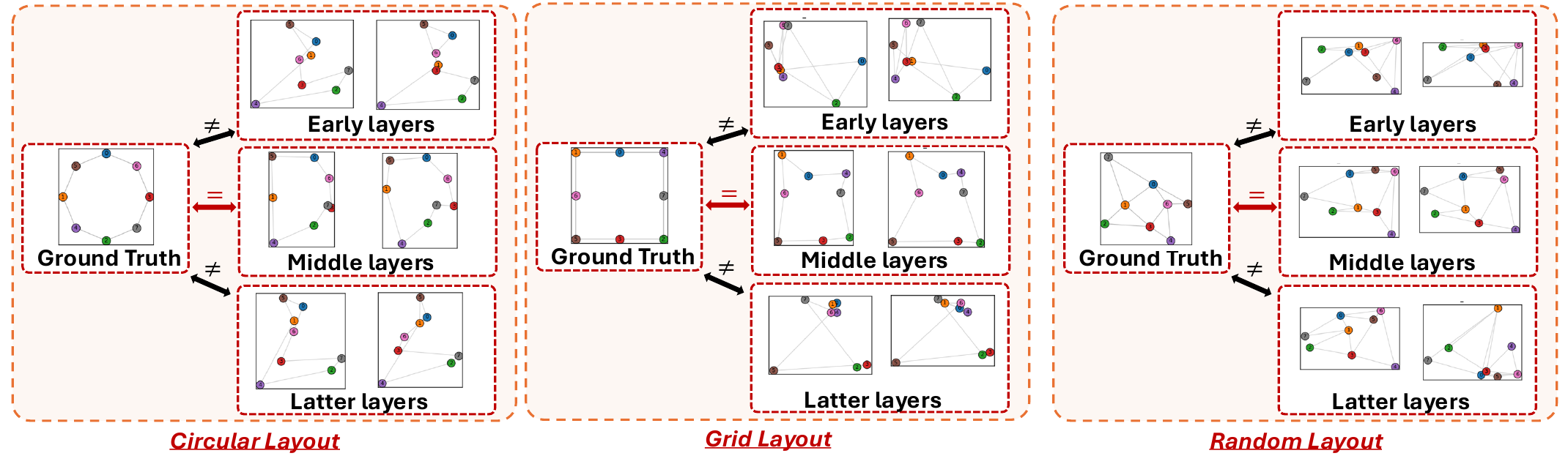}
	\vspace{-0.2in}
	\caption{Visualization of extracted spatial feature in different layers of VLM}
	\vspace{-0.1in}
	\label{fig:pca_visualized}
\end{figure}

\vspace{-0.05in}
\subsection{Results of Linear Spatial Feature Extraction}
\vspace{-0.05in}


\paragraph{Qualitative visualization of the spatial subspace:}
\label{subsec:pca}

To visualize the geometry in the spatial subspace, we collect the extracted per-object activations $\{\tilde h^{(s,o)}_\ell\}$ in different scenes 
and run PCA:
\begin{equation}
	\tilde{\mathbf H}_\ell \;=\; U \Sigma V^\top,
	\qquad
	Z \;=\; \tilde{\mathbf H}_\ell \,V_{:,1:3}\;\in\;\mathbb{R}^{N\times 3}.
	\label{eq:pca}
\end{equation}
Each row of $Z$ is the 3D PCA embedding of one object-token in one scene and $N{=}\sum_s m^{(s)}$. Figure \ref{fig:pca_visualized} plots $Z$ colored by the object's true 3D position, and demonstrates that with our linear spatial extraction, the resulting spatial subspace neatly unfolds at the middle layers of the VLM and mirrors the physical scene layout. More visualization examples are in Appendix \ref{app:v_eg}.


\vspace{-0.05in}
\paragraph{Quantitative metrics of topology similarity:}
\label{subsec:metrics}

Besides the qualitative visualization above, we also measure topology-faithfulness of $\tilde H_\ell^{(s)}$ to the scene's true 3D coordinates $X^{(s)}$ using Dirichlet energy \cite{belkin2003laplacian,park2024iclr}. We define a Gaussian-kernel graph over the scene's objects as $W_{ij} = \kappa_\tau(x_i,x_j) = \exp\Bigl(-\frac{\|x_i-x_j\|^2}{2\tau^2}\Bigr)$.
With the degree matrix $D_{ii}=\sum_j W_{ij}$ and the graph Laplacian $L = D - W$, the Dirichlet energy of an object matrix $\tilde H_\ell^{(s)}$ is
\begin{equation}
	\mathcal{E}_X(\tilde H_\ell^{(s)})
	\;=\; \frac{1}{2}\sum\nolimits_{i,j} W_{ij}\,\|\tilde h^{(s,o_i)}_\ell - \tilde h^{(s,o_j)}_\ell\|^2
	\;=\; \mathrm{tr}\!\bigl(\tilde H_\ell^{(s)\,\top}\, L\,\tilde H_\ell^{(s)}\bigr).
	\label{eq:dirichlet}	
\end{equation}

\begin{figure}[h]
	\centering
	\vspace{-0.05in}
	\begin{subfigure}[t]{0.4\linewidth}
		\centering
		\includegraphics[width=\textwidth]{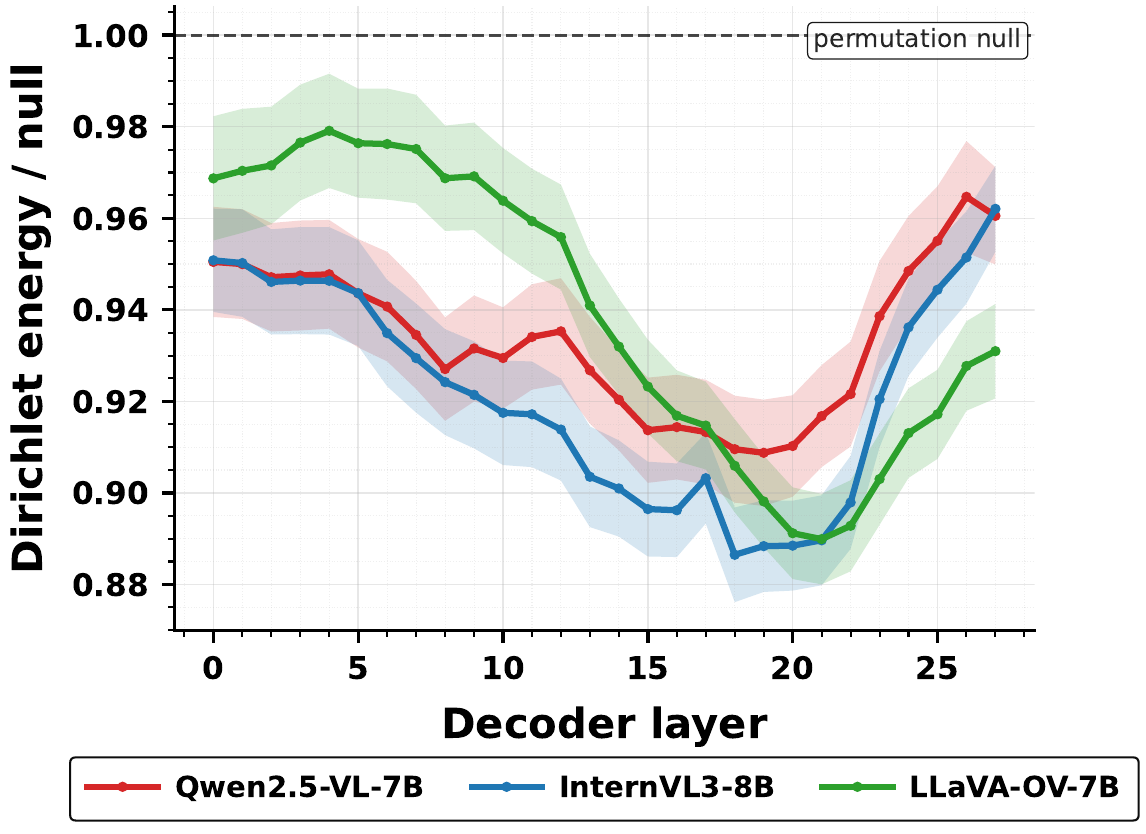}
		\caption{Dirichlet energy in different layers}
		\label{fig:dirichlet_energy_layers}
	\end{subfigure}
	\begin{subfigure}[t]{0.38\linewidth}
		\centering
		\includegraphics[width=\textwidth]{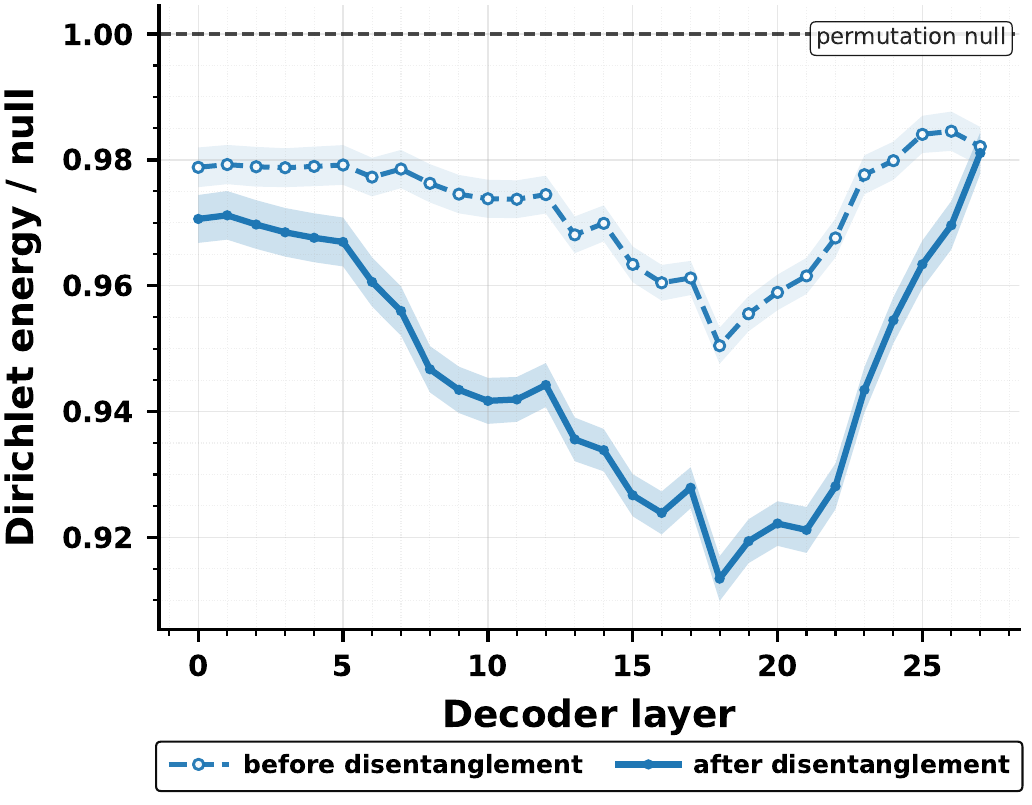}
		\caption{Effectiveness of disentanglement}
		\label{fig:disentanglement_comparison}
	\end{subfigure}
	\vspace{-0.05in}
	\caption{Dirichlet energy in different VLM layers}
	\vspace{-0.1in}
	\label{fig:dirichlet_energy_layers}
\end{figure}

A low Dirichlet energy indicates that closer objects in the physical 3D space are mapped to proximal points in the latent space, hence better preserving the scene's physical topology. To normalize against scale and per-scene difficulty, we report the \emph{Dirichlet ratio} against a permutation null, defined as $\rho^{\rm dir}(\tilde H_\ell^{(s)})= {\mathcal{E}_X(\tilde H_\ell^{(s)})}/{\mathbb{E}_{\Pi\sim\mathcal{U}(\mathfrak{S}_m)}\!\bigl[\mathcal{E}_X(\Pi\,\tilde H_\ell^{(s)})\bigr]}$, estimated with 1{,}000 row-shuffles of $\tilde H_\ell^{(s)}$ scene. $\rho^{\rm dir} < 1$ indicates the activation is more topology-faithful than a random binding of objects to rows. 


Being consistent with the visualizations in Figure \ref{fig:pca_visualized}, Figure \ref{fig:dirichlet_energy_layers} shows that the Dirichlet energy is minimized at these intermediate depths. Combining these findings with prior literature on model circuits and feature binding, we deduce a distinct \textbf{functional division} across the network: \emph{early layers} extract low-level visual features, \emph{middle layers} construct a spatially aware 3D understanding, and \emph{late layers} engage in linguistic reasoning to generate textual outputs. Furthermore, we validate the effectiveness of our disentanglement approach by comparing the Dirichlet energy before and after extraction. As shown in Figure \ref{fig:disentanglement_comparison}, the spatial localization within the middle layers becomes substantially more pronounced after the spatially irrelevant semantics are projected out.

\vspace{-0.1in}
\paragraph{Casual verification.}
To verify that the VLM actively uses the extracted spatial-aware subspace, we performed an casual intervention. As shown in Figure \ref{fig:casual_verification}, by injecting a controlled vector along the spatial subspace's principal axis into an object's activation, the 3D-coordinate probe's prediction shifted linearly and selectively along that targeted axis. In contrast, applying a matched control vector in a null direction produced no detectable effect. This confirms that the extracted subspace represents the actual 3D scene topology, and the model relies on this subspace for its internal 3D readout, rather than arbitrary representations. More details can be found in Appendix \ref{app:steering}.

\begin{figure}[h]
	\centering
	\vspace{-0.05in}
	\includegraphics[width=2.2in]{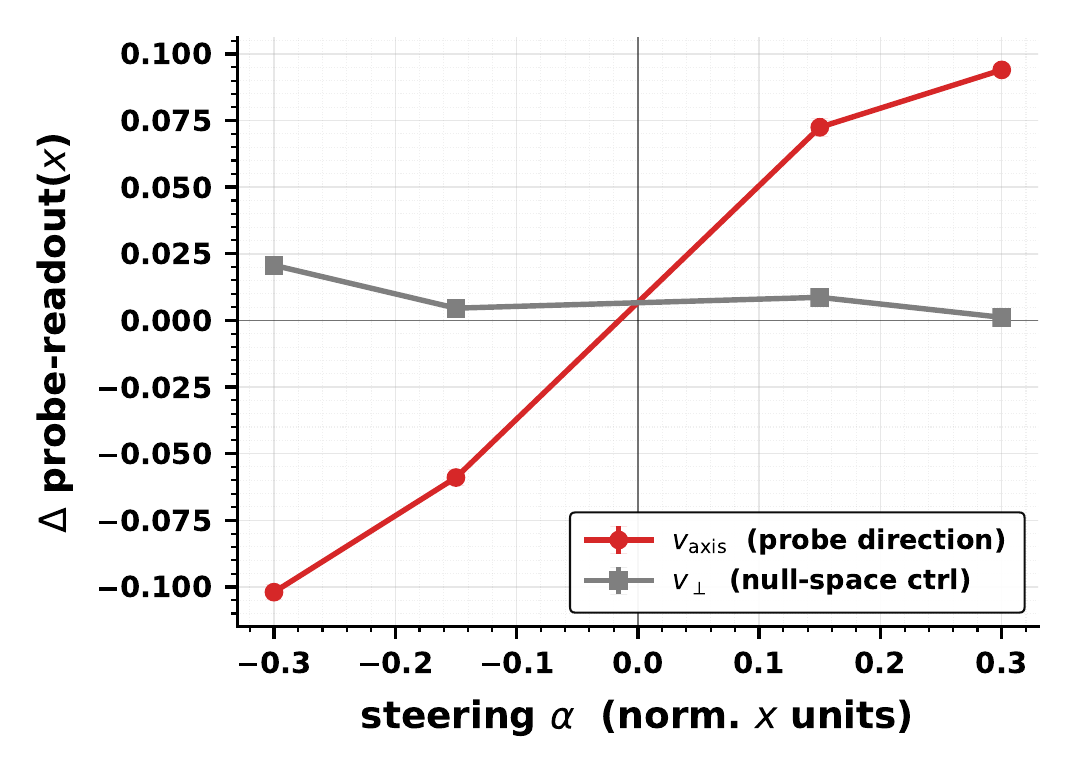}
	\vspace{-0.1in}
	\caption{Causal verification on the VLM's use of the spatial subspace}
	\vspace{-0.05in}
	\label{fig:casual_verification}
\end{figure}

\vspace{-0.1in}
\section{Shaping the Spatial Subspace via Dirichlet-Energy Regularization}
\vspace{-0.1in}
\label{sec:enforce}

With the extracted spatial subspace, we further explore possibilities of deliberately shaping it to enhance the VLM's spatial reasoning capabilities. In this section, we mathematically formalize this spatial subspace to derive a quantitative penalty. We then demonstrate that enforcing this penalty in VLM fine-tuning yields measurable and theoretically guaranteed gains on spatial reasoning.


\vspace{-0.05in}
\subsection{The Target Representation of Laplacian Eigenmaps}
\vspace{-0.05in}
\label{subsec:theory}

We first provide a mathematical formulation for the spatial subspace extracted in Section \ref{sec:extract}, to motivate why the \emph{Dirichlet energy} is the ideal scalar to minimize in VLM fine-tuning. Eq.~ (\ref{eq:dirichlet}) described the Dirichlet energy of an object-token matrix $H \in \mathbb{R}^{m\times d}$ on the scene's 3D Gaussian-kernel graph $W$ is $\mathcal{E}_X(H) = \mathrm{tr}(H^\top L H)$. Geometrically, a low Dirichlet energy ensures that rows of $H$ corresponding to spatially nearby objects (large $W_{ij}$) are themselves nearby in $\mathbb{R}^d$. To formalize this, we seek the exact mathematical structure of the representation $H^*$ that minimizes this energy.

\begin{theorem}[Dirichlet minimization $\Rightarrow$ Laplacian eigenmap PCA]
	\label{thm:eigenmap}
	Let $X \in \mathbb{R}^{m \times 3}$ be the 3D coordinates of $m$ scene objects, $W_{ij} = \kappa_\tau(x_i, x_j)$ their Gaussian-kernel graph, and $L = D - W$ the unnormalized graph Laplacian. 
	Fix $s \le \min(m,d) - 1$ and consider the constrained optimization:
	\begin{equation}
		H^* \;\in\; \arg\min_{H \in \mathbb{R}^{m \times d}} \mathcal{E}_X(H) \quad \text{subject to} \quad \sigma_k(H) \ge \epsilon_k \quad \text{for all } k \in[s],
	\end{equation}
	where $\sigma_k(H)$ is the $k$-th singular value of $H$ and $\epsilon_1 > \epsilon_2 > \dots > \epsilon_s > 0$ are constants.
	Then:
	\vspace{-0.05in}
	\begin{enumerate}[leftmargin=0.25in]
		\item[(a)] The minimizer is unique up to right-orthogonal transformation. The left singular vectors satisfy $u_k(H^*) = z^{(k)}$ for $k=1, \dots, s$, where $z^{(k)}$ is the $k$-th eigenvector of $L$ ordered by smallest eigenvalues.
		\vspace{-0.05in}
		\item[(b)] After mean-centering $H_c^* \triangleq (I - \frac{1}{m}\mathbf{1}\mathbf{1}^\top)H^*$, the $k$-th principal component (left singular vector) of $H_c^*$ exactly recovers $z^{(k+1)}$ for $k=1, \dots, s-1$.
	\end{enumerate}
	\vspace{-0.1in}
\end{theorem}

This theorem proves that if we force a representation to maintain some variance (the $\epsilon_k$ constraints prevent trivial collapse to a single point) while minimizing the Dirichlet energy, the optimal representation $H^*$ will naturally align its principal components with the eigenvectors of the scene's graph Laplacian. In practical VLM fine-tuning, the standard language loss ($\mathcal{L}_{\rm LM}$) implicitly provides this non-degeneracy constraint, as the network cannot collapse its hidden states to zero without destroying its ability to predict text tokens.

\vspace{-0.1in}
\paragraph{Proof Sketch.}
The full formal proof is in Appendix~\ref{app:proofs} and relies on the Ky Fan trace minimum theorem. (1) By expanding $H$ into its Singular Value Decomposition ($H=U\Sigma V^\top$), the Dirichlet objective exactly decomposes along the orthonormal left-singular directions $u_k$. (2) The lower-bound constraints $\sigma_k \ge \epsilon_k$ freeze the singular values. (3) Applying the weighted Ky Fan theorem proves that the optimal orthonormal vectors $u_k$ must perfectly align with the eigenvectors of $L$ corresponding to its smallest eigenvalues. (4) Mean-centering simply projects out the trivial uniform mode $z^{(1)} \propto \mathbf{1}$.

This identifies the representation as a set of abstract graph eigenvectors. To connect back to physical vision, we consider the continuous limit:

\begin{theorem}[Continuous Limit to 3D Space]
	\label{thm:limit}
	Let $\{x_i\}_{i=1}^m$ be i.i.d. samples from a smooth density $\rho$ on a compact Riemannian manifold $\Omega \subset \mathbb{R}^3$. Let $L^{(\tau)}$ be the Gaussian-kernel Laplacian with bandwidth $\tau_m$. If $\tau_m \to 0$ and $m\tau_m^{3+\alpha} \to \infty$, then for any fixed $k$, with probability tending to 1, the discrete eigenvector $z^{(k+1)}(x_i)$ converges uniformly to $\rho(x_i)^{-1/2}\phi^{(k)}(x_i)$, where $\phi^{(k)}$ is the $k$-th non-constant eigenfunction of the weighted Laplace-Beltrami operator $-\Delta_\rho$. For uniform $\rho$ on a Euclidean cube, $\phi^{(1)}, \phi^{(2)}, \phi^{(3)}$ are exactly the Cartesian coordinates $x, y, z$ (up to a smooth cosine reparameterization).
	\vspace{-0.05in}
\end{theorem}

The vectors extracted by Theorem~\ref{thm:eigenmap} correspond directly to physical 3D space. While real-world scenes feature more complex object densities, our probing and training datasets utilize uniformly distributed objects, hence allowing the Laplacian eigenmaps to naturally recover clean Cartesian axes.

\vspace{-0.05in}
\subsection{Loss Formulation and Realizability}
\vspace{-0.05in}
\label{subsec:loss}

Motivated by that Dirichlet energy uniquely targets true 3D spatial axes, we introduce a training intervention that shapes the latent space toward this target in VLM fine-tuning. We add a single scalar Dirichlet-energy term to the standard LoRA cross-entropy loss at the cognitive-map layer $\ell$:
\begin{equation}
	\mathcal{L}(\theta) = \mathcal{L}_{\rm LM}(\theta)
	\;+\; \lambda_{\rm dir}\cdot \mathcal{E}_X\bigl(P_\perp H_\ell(\theta)\bigr),
\end{equation}
where $P_\perp$ is the cross-scene spatial-feature-extraction projector from Section \ref{sec:extract} and $W$ is the scene's 3D Gaussian-kernel graph. 

Why does forcing the model into this specific representation improve downstream spatial reasoning? We prove that it guarantees the \emph{expressivity} required by the LM head:

\begin{theorem}[Realizability of Spatial Readouts]
	\label{thm:realizability}
	Let $f: \mathbb{R}^3 \to \mathbb{R}$ be any linear function of the monotonic coordinate transformations (e.g., an axis-aligned ordinal comparison encoding ``is object A to the left of object B''). For the Dirichlet-minimizing representation $H^*$ satisfying Theorem~\ref{thm:limit}, and for every $\eta > 0$, there exists a linear functional $\ell : \mathbb{R}^d \to \mathbb{R}$ such that:
	\begin{equation}
		\sup\nolimits_{i \in[m]} |\ell(h^*_i) - f(x_i)| \;<\; \eta \quad \text{with high probability as } m \to \infty.
	\end{equation}
\end{theorem}
Since the latent space is reshaped to mirror the 3D world coordinates, spatial relational questions become linearly separable. The VLM's linear language modeling head possesses the exact architectural capacity required to answer perfectly without needing deeper, non-linear geometric parsers.

\begin{table}
	\small
	\centering
	
	\begin{tabular}{lcccc}
		\toprule
		\multirow{2}{*}{Overall accuracy (\%)} & \multicolumn{2}{c}{Qwen2.5-VL-7B} & \multicolumn{2}{c}{InternVL3-8B} \\
		\cmidrule(lr){2-3} \cmidrule(lr){4-5}
		& VSI-bench & MindCube & VSI-bench & MindCube \\
		\midrule
		Naive LoRA fine-tuning                    & 29.3 & 41.3 & 36.7 & 45.2 \\
		2D regularization \cite{kang2026linear}               & 28.7 & 39.8 & 36.5 & 44.7 \\
		Text-domain cog map \cite{wang2025mindcube}        & 28.8 & 38.5 & 36.0 & 43.1\\
		\textbf{Ours}      & \textbf{\underline{35.6}} & \textbf{\underline{48.3}} & \textbf{\underline{42.4}} & \textbf{\underline{51.7}} \\
		\bottomrule
	\end{tabular}
	\vspace{0.05in}
	\caption{Overall accuracy on spatial tasks}
	\label{tab:overall}
	\vspace{-0.25in}
\end{table}

\vspace{-0.05in}
\subsection{Experiment Setup}
\vspace{-0.05in}
\label{subsec:setup}

To experimentally validate the efficacy of our proposed VLM training intervention, we utilize our proposed loss to fine-tune pretrained VLMs, namely Qwen2.5-VL-7B \cite{qwen2.5-VL} and InternVL3-8B \cite{zhu2025internvl3}, using the SynSpat3D dataset introduced in Section \ref{sec:diagnose}, and evaluate zero-shot transfer to the following real-world spatial reasoning benchmarks:
\vspace{-0.1in}
\begin{itemize}[leftmargin=0.1in]
	\item \textbf{VSI-Bench \cite{yang2025thinking}}: a comprehensive benchmark for models' capabilities in fine-grained 3D spatial reasoning, relational judgments, and absolute distance estimation from egocentric video sequences. The evaluated tasks include: object counting (\textit{obj-count}), object and room size estimation (\textit{obj-size}, \textit{room-size}), absolute and relative distance estimation (\textit{abs-dist}, \textit{rel-dist}), relative direction across varying difficulties (\textit{rel-d}), appearance order (\textit{appear-ord}), and route planning (\textit{route}).
	\item \textbf{MindCube \cite{wang2025mindcube}}: a benchmark designed to measure how well models can construct internal cognitive maps and perform spatial mental modeling given only restricted and partial 2D views.
\end{itemize}
\vspace{-0.05in}

We train LoRA (rank-16) on the attention matrices for 500 steps. Trainable parameters are restricted strictly to the middle ``spatial-aware'' layers where the Dirichlet loss is applied, leaving the early visual binding and late linguistic reasoning layers untouched. We empirically identify these middle layers by selecting those with the lowest pre-intervention Dirichlet energy ratio, as established in our diagnostic probing (Figure \ref{fig:dirichlet_energy_layers}). We sweep $\lambda_{\rm dir}$ to observe its effect, reporting the optimal settings in our main results and exploring the trade-off in Section \ref{subsec:ablation}. All results are averaged over $n=4$ random seeds. We compare our approach with the following baselines: 

\vspace{-0.1in}
\begin{itemize}[leftmargin=0.1in]
\item \textbf{Naive LoRA fine-tuning}, which relies exclusively on the standard language modeling loss without any topological regularization.
\item \textbf{2D spatial regularization} \cite{kang2026linear}, which supervises the intermediate layers' spatial identifiers based on 2D pixel-space distances rather than true 3D geometry.
\item \textbf{Text-domain cognitive map prediction} \cite{wang2025mindcube}, which explicitly fine-tunes the model to generate the scene's spatial coordinates directly in the textual output space.
\end{itemize}
\vspace{-0.05in}

We do not include comparisons against more advanced baselines, such as those incorporating external spatial modalities \cite{wang2025mindcube,wang2026mosaicthinker,cheng2024spatialrgpt} or reinforcement learning \cite{batra2025spatialthinker,liu2025spatial,wu2025reinforcing}, because our latent regularization is fundamentally orthogonal to and naturally complement these approaches.

\begin{table*}[h]
	\centering
	\vspace{-0.05in}
	\resizebox{\linewidth}{!}{
		\begin{tabular}{l *{11}{c}}
			\toprule
			\textbf{Task accuracy (\%)} & 
			\makecell{\textbf{over-} \\ \textbf{all}} & 
			\makecell{\textbf{obj-} \\ \textbf{count}} & 
			\makecell{\textbf{obj-} \\ \textbf{size}} & 
			\makecell{\textbf{room-} \\ \textbf{size}} & 
			\makecell{\textbf{abs-} \\ \textbf{dist}} & 
			\makecell{\textbf{rel-d-} \\ \textbf{easy}} & 
			\makecell{\textbf{rel-d-} \\ \textbf{medium}} & 
			\makecell{\textbf{rel-d-} \\ \textbf{hard}} & 
			\makecell{\textbf{rel-} \\ \textbf{dist}} & 
			\makecell{\textbf{appear-} \\ \textbf{ord}} & 
			\textbf{route} \\
			\midrule
			
			
			Naive LoRA fine-tuning & 
			29.3 & 
			7.4 & 
			26.3 & 
			29.3 & 
			19.5 & 
			47.9 & 
			37.1 & 
			24.5 & 
			30.4 & 
			37.3 & 
			33.6\\
			\addlinespace
			
			\makecell[l]{2D regularization \cite{kang2026linear}} & 
			\valdown{28.7}{0.6} & 
			\valup{7.5}{0.1} & 
			\valdown{26.0}{0.3} & 
			\valdown{25.1}{4.2} & 
			\valdown{19.1}{0.4} & 
			\valdown{46.4}{1.5} & 
			\valdown{37.0}{0.1} & 
			\valdown{24.0}{0.5} & 
			\valup{31.2}{0.8} & 
			\valup{37.6}{0.4} & 
			\valdown{32.9}{0.8} \\
			\addlinespace
			
			\makecell[l]{Text-domain \\cog map \cite{wang2025mindcube}} & 
			\valdown{28.8}{0.5} & 
			\valdown{6.7}{0.7} & 
			\valup{26.6}{0.3} & 
			\valdown{27.4}{1.8} & 
			\valdown{18.7}{0.8} & 
			\valdown{47.4}{0.5} & 
			\valup{37.3}{0.2} & 
			\valdown{23.4}{1.1} & 
			\valup{30.8}{0.4} & 
			\valup{38.1}{0.8} & 
			\valdown{31.2}{2.4} \\
			\addlinespace
			
			
			\makecell[l]{\textbf{Ours}} & 
			\valup{35.6}{6.3} & 
			\valdown{7.2}{0.2} & 
			\valup{31.3}{5.0} & 
			\valup{33.4}{4.1} & 
			\valup{28.9}{9.4} & 
			\valup{60.0}{12.1} & 
			\valup{45.5}{8.4} & 
			\valup{37.2}{6.8} & 
			\valup{35.8}{5.4} & 
			\valup{41.5}{4.2} & 
			\valup{37.7}{4.1} \\
			
			\bottomrule
		\end{tabular}
	}
	\vspace{-0.05in}
	\caption{Per-task accuracy on VSI-Bench using Qwen2.5-VL-7B. Colored values indicate absolute percentage changes relative to the case of LM loss only.}
	\label{tab:qwen_vsi}
\end{table*}

\vspace{-0.15in}
\subsection{Main Results and Sample Complexity}
\vspace{-0.05in}
\label{subsec:results}

\begin{table*}[h]
	\small
	\centering
	\vspace{-0.05in}
	\label{tab:intern_vsi}
	\resizebox{\linewidth}{!}{
		\begin{tabular}{l *{11}{c}}
			\toprule
			\textbf{Task accuracy (\%)} & 
			\makecell{\textbf{over-} \\ \textbf{all}}& 
			\makecell{\textbf{obj-} \\ \textbf{count}} & 
			\makecell{\textbf{obj-} \\ \textbf{size}} & 
			\makecell{\textbf{room-} \\ \textbf{size}} & 
			\makecell{\textbf{abs-} \\ \textbf{dist}} & 
			\makecell{\textbf{rel-d-} \\ \textbf{easy}} & 
			\makecell{\textbf{rel-d-} \\ \textbf{medium}} & 
			\makecell{\textbf{rel-d-} \\ \textbf{hard}} & 
			\makecell{\textbf{rel-} \\ \textbf{dist}} & 
			\makecell{\textbf{appear-} \\ \textbf{ord}} & 
			\textbf{route} \\
			\midrule
			
			
			Naive LoRA fine-tuning & 
			36.7 & 
			10.9 & 
			28.4 & 
			43.1 & 
			80.8 & 
			49.8 & 
			33.9 & 
			26.1 & 
			26.5 & 
			34.0 & 
			33.6\\
			\addlinespace
			
			\makecell[l]{2D regularization} & 
			\valdown{36.5}{0.2} & 
			10.9 & 
			\valup{28.8}{0.4} & 
			\valup{43.5}{0.2} & 
			\valup{81.3}{0.5} & 
			\valdown{47.4}{2.4} & 
			\valdown{33.3}{0.6} & 
			\valdown{26.0}{0.1} & 
			\valup{26.8}{0.3} & 
			34.0 & 
			\valdown{33.5}{0.1} \\
			\addlinespace
			
			\makecell[l]{text-domain \\cog-map pred} & 
			\valdown{36.0}{0.7} & 
			\valdown{9.2}{1.7} & 
			\valup{28.5}{0.1} & 
			\valdown{40.6}{2.5} & 
			\valdown{78.9}{2.9} & 
			49.8 & 
			\valdown{33.6}{0.3} & 
			\valdown{25.7}{0.4} & 
			\valup{26.6}{0.1} & 
			\valup{34.3}{0.3} & 
			\valdown{32.3}{1.3} \\
			\addlinespace
			
			
			\makecell[l]{\textbf{Ours}} & 
			\valup{42.4}{5.7} & 
			\valup{11.8}{0.9} & 
			\valdown{28.3}{0.1} & 
			\valup{49.3}{6.2} & 
			\valup{89.1}{8.3} & 
			\valup{60.7}{10.9} & 
			\valup{43.4}{9.5} & 
			\valup{29.4}{7.3} & 
			\valdown{33.4}{0.2} &
			\valup{38.4}{4.4} & 
			\valup{36.8}{3.2} \\
			
			\bottomrule
		\end{tabular}
	}
	\vspace{-0.05in}
	\caption{Per-task accuracy on VSI-Bench using InternVL3-8B. Colored values indicate absolute percentage changes relative to the case of LM loss only.}
	\vspace{-0.2in}
	\label{tab:intern_vsi}
\end{table*}

As shown in Tables~\ref{tab:overall}, \ref{tab:qwen_vsi}, and \ref{tab:intern_vsi}, our proposed Dirichlet-energy regularization consistently outperforms all the baselines. The 2D regularization method fails to yield improvements because it constrains the model to rely on 2D pixel-space distances rather than true 3D geometry. Similarly, the text-domain cognitive map prediction falls short because it forces the language model to over-index on text semantics rather than faithfully grounding the physical visual features. Analyzing the per-task accuracy reveals that our method significantly excels in categories requiring a robust understanding of scene topology, such as relative direction (\textit{rel-d-easy}, \textit{rel-d-medium}, \textit{rel-d-hard}) and absolute distance (\textit{abs-dist}). Conversely, for tasks such as \textit{route} and \textit{appear-ord}, the gains of our method become smaller, as these tasks demand reasoning over temporal dynamics and sequential planning, which fall outside the scope of our purely spatial topological regularization.

Crucially, fine-tuning requires just 500 steps on simple synthetic data to generalize to complex, real-world benchmarks. We theoretically ground this extreme data efficiency:

\begin{theorem}[Sample-Complexity Reduction]
	\label{thm:sample_complexity}
	Let a spatial task predict $y = w^{\star \top}h + \xi$, where $w^\star$ depends only on the world-coordinate axes (spanning the top 3 PCs of $H^*$). For the empirical risk minimizer $\hat{w}_N$ over $N$ samples, with probability $\ge 1-\delta$, the excess risk is bounded by:
	\begin{equation}
		\mathbb{E}\big[(\hat{w}_N^\top h - w^{\star \top}h)^2\big] \;\le\; \frac{3C \sigma^2_\xi}{N} \log(1/\delta),
	\end{equation}
	improving upon the standard $O(d/N)$ bound of an unstructured baseline.
\end{theorem}

\vspace{-0.05in}
Standard VLM residual streams are high-dimensional ($d \approx 4000$), typically requiring massive datasets to prevent overfitting. Theorem~\ref{thm:sample_complexity} proves that the Dirichlet penalty collapses task-irrelevant variance, shrinking the effective dimension to exactly 3. This reduces sample complexity by a factor of $d/3$, enabling the LM head to learn highly generalizable spatial readouts from minimal data.

\vspace{-0.05in}
\subsection{Ablation Studies}
\vspace{-0.05in}
\label{subsec:ablation}

\begin{wraptable}{R}{0.45\textwidth}
	\centering
	\vspace{-0.2in}
	\vspace{-0.1in}
	\fontsize{8}{9}\selectfont
	\begin{tabular}{lcc}
		\toprule
		$\lambda$ & Qwen2.5-VL-7B & InternVL3-8B \\
		\midrule
		0   & 29.3 & 36.7 \\
		0.3 & 29.5 & 39.1 \\
		1   & 32.3 & \textbf{42.4} \\
		3   & \textbf{35.6} & 37.2 \\
		9   & 28.6 & 35.8 \\
		\bottomrule
	\end{tabular}
	\caption{Ablation study on the Dirichlet regularization weight ($\lambda$) on VSI-Bench}
	\vspace{-0.15in}
	\label{tab:lambda_ablation}
\end{wraptable}

To evaluate the sensitivity of our regularization method, we vary the Dirichlet penalty weight $\lambda$ during the 500-step fine-tuning. As shown in Table~\ref{tab:lambda_ablation}, zero-shot performance on VSI-Bench peaks at moderate values ($\lambda=3$ for Qwen, $\lambda=1$ for InternVL). However, excessive regularization ($\lambda=9$) drastically degrades accuracy, even dropping below the unregularized baseline. 

We formalize this exact trade-off mathematically. While a small penalty cleanses the spatial subspace, an overly aggressive $\lambda$ distorts the residual stream, harming VLM's language modeling capabilities:

\begin{theorem}[Training-Time Risk Decomposition]
	\label{thm:risk}
	Let $\mathcal{L}_\lambda(\theta) = \mathcal{L}_{\rm LM}(\theta) + \lambda \mathcal{R}_X(H_\theta)$, where $\mathcal{R}_X$ is the Dirichlet ratio. Let $\theta^*$ be the population minimizer of the standard LM loss, and $\theta^*_\lambda$ the minimizer of the regularized loss. For a constant $\beta > 0$ correlating the spatial risk surface with the Dirichlet energy, the population spatial-task risk $R_{\rm spatial}$ satisfies:
	\begin{equation}
		R_{\rm spatial}(\theta^*_\lambda) \;\le\; R_{\rm spatial}(\theta^*) \;-\; \lambda \cdot \beta \, (\mathcal{R}_X(H_{\theta^*}) - \mathcal{R}_X^*) \;+\; O(\lambda^2).
	\end{equation}
\end{theorem}

Because decreasing Dirichlet energy inherently improves spatial realizability (Theorem~\ref{thm:realizability}), their gradients align ($\beta > 0$). Consequently, the linear term strictly guarantees downstream spatial VQA improvements for moderate $\lambda$. Conversely, the $O(\lambda^2)$ term bounds the over-regularization risk, perfectly capturing the empirical performance inversion observed at $\lambda=9$.

\vspace{-0.1in}
\section{Conclusion}
\vspace{-0.1in}
\label{sec:discussion}

In this paper, we isolate the latent 3D spatial subspace within VLMs and introduce a mathematically grounded Dirichlet-energy regularizer to directly shape it. By demonstrating that a minimal, topology-preserving fine-tuning stage significantly enhances zero-shot spatial reasoning performance, we provide a principled methodology for developing VLMs with robust, intrinsic cognitive maps.

\setcitestyle{numbers}
\bibliographystyle{unsrt}
\bibliography{main}

\newpage
\appendix

\section{Data generation details}
\label{app:data_gen}

This appendix expands on the two-step synthetic-scene pipeline summarized in \S\ref{sec:diagnose} and used both for probing (\S\ref{sec:diagnose}--\S\ref{sec:extract}) and for Dirichlet-regularized fine-tuning (\S\ref{sec:enforce}). 

\paragraph{Step 1: 3D scene generation.} Each scene places between $3$ and $8$ rigid objects (corpus default; the looser ``$6$--$12$'' range cited in the main text covers the broader set of pretraining variants) on a flat floor inside a bounded working volume of $[-4,\,4]\!\times\![-4,\,4]$ in world units, with $z$ pinned to the object's radius so all objects rest on the ground plane. The independent latent factors per object are: \emph{shape} drawn uniformly from $\{\textit{cube},\textit{sphere},\textit{cylinder}\}$, \emph{color} drawn uniformly from an $8$-way palette $\{\textit{red},\textit{green},\textit{blue},\textit{yellow},\textit{cyan},\textit{magenta},\textit{orange},\textit{purple}\}$, \emph{size} drawn from $\{0.4,\,0.6,\,0.8\}$ world units, and \emph{position} sampled by rejection until pairwise edge-to-edge separation exceeds $0.2$ world units (no overlapping objects, no occluded clusters). Color and shape are sampled independently of position, so \emph{at the population level the visual identity of an object carries no information about where it is}; this independence is what makes the variance gap reported in Fig.~\ref{fig:attention_dominance} a property of the model rather than of the data. Each scene receives a unique hash-based ID (\texttt{s\_<hash>}) and is written to disk as \texttt{scene.json} containing the full per-object table $\{\textit{name}, \textit{shape}, \textit{color}, \textit{size}, \textit{world\_position}\}$.

\paragraph{Step 2: egocentric frame rendering.} For each canonical scene we render $T{=}16$ frames at $448{\times}448$ resolution under one of three camera regimes. (i) \emph{Orbit} ($1$-DoF planar arc): the eye traces an $180^\circ$ arc at fixed altitude (sampled from $\{3.5,\,4.5\}$) and radius (sampled from $\{8.0,\,9.0\}$), with the look-at point fixed at $z{=}0.5$ in the scene center. (ii) \emph{Free6DoF} (the default for all probing/training runs reported in the body): both the eye and the look-at point drift in $\mathbb{R}^3$, the camera rolls, and the orbit radius/altitude jitter independently of arc angle (\texttt{eye\_jitter}=1.2, \texttt{radius\_jitter}=1.0, \texttt{altitude\_jitter}=0.8, \texttt{target\_jitter}=1.5; smooth-noise basis with $3$ frequency modes). (iii) \emph{Person-walk} (used as a generalization control in §\ref{sec:diagnose} and Appendix~\ref{app:diagnose}): instead of orbiting around the working volume, the camera is placed \emph{inside} it at human eye-height ($z\!=\!1.5$) and walks forward along a smooth-noise path with stochastic per-frame yaw and pitch. The dynamics are: forward step length sampled around $\mathtt{speed\_mean}=0.5$ world units with $\mathtt{speed\_amp}=0.25$ smooth-noise modulation; per-frame yaw drift up to $\pm 25^\circ$; per-frame pitch up to $\pm 12^\circ$; collision avoidance with $\mathtt{object\_margin}=0.6$ around every object and $\mathtt{bounds\_margin}=0.3$ from the working-volume edges; up to $40$ retries on the whole trajectory if a step cannot be repaired without violating the margins. This regime produces partial-view egocentric trajectories where individual frames frequently see zero or one object, so the corresponding spatial reasoning task requires \emph{cross-frame} integration (the body's ``cognitive-map'' setting). Across all three regimes, each frame is post-validated to ensure (orbit / free6dof) at least one object remains visible or (person-walk) the trajectory satisfies its margin constraints; failed frames trigger trajectory repair. The horizontal field of view is $60^\circ$ and the background is uniform mid-gray (RGB $155$). To dissociate the scene layout from any single trajectory, every scene is re-rendered under $4$ independent trajectories (indexed \texttt{traj\_idx} $\in\{0,1,2,3\}$); we use \texttt{traj\_idx}=$1$--$2$ for probing and Dirichlet training.

\paragraph{Spatial QA generation.} Per scene, $5$ questions are programmatically synthesized from the ground-truth $\{\textit{name}, \textit{world\_position}\}$ table. The two question kinds used for Dirichlet training are: (a) \emph{Relative-position binary} --- ``Is the \textsc{[color] [shape]} \textsc{[in front of/behind/to the left of/to the right of]} the \textsc{[color] [shape]}?'' answered against the camera-aligned $xy$-plane of the first frame; (b) \emph{Distance-order ternary} --- ``Which is closer to the \textsc{[reference]}, the \textsc{[A]} or the \textsc{[B]}?'' Distractor objects are guaranteed present in the rendered frames but not always in the question. Each question stores the in-question object names alongside their $3$D coordinates so the same JSONL can drive both standard LM cross-entropy and the Dirichlet-energy regularizer (which reads the coordinates directly).

\paragraph{Probing dataset (Free6DoF $1{,}000$ scenes).} The headline probing experiments in \S\ref{sec:diagnose} draw from a $1{,}000$-scene Free6DoF corpus, each with $5$ questions $\to 5{,}000$ video-QA pairs, plus extended controls (\emph{Circle8} and \emph{Grid3$\times$3} layouts, $\sim$$5{,}000$ scenes total) used for the topology-vs-shortcut robustness checks reported in Appendix~\ref{app:diagnose}. The $400$-scene subsets used for layer-wise Dirichlet-energy curves (Fig.~\ref{fig:dirichlet_energy_layers}) and the residualized variant (Fig.~\ref{fig:disentanglement_comparison}) are the first $400$ scenes by scene-ID hash so the same split is reused across all backbones.

\paragraph{Counterfactual perturbation set (\S\ref{sec:diagnose}).} For the color/position swap experiment we draw $50$ base scenes (a random subset of the probing corpus) and produce three variants per scene by re-rendering with the same trajectory but: (i) \emph{original}; (ii) \emph{color-swap} --- the $k$ assigned colors of the $k$ objects are permuted by a random derangement, leaving every $3$D position untouched; (iii) \emph{position-swap} --- the $k$ positions are permuted, leaving every color untouched. We then ask the same $\sim$$9$ spatial questions per scene, yielding $1{,}290$ QA pairs per seed used in the flip-rate analysis.

\paragraph{Train/val splits for Dirichlet fine-tuning.} The fine-tuning corpus at \texttt{data/dirichlet\_train/} contains $2{,}988$ training and $332$ validation video-QA pairs (\texttt{train.jsonl}/\texttt{val.jsonl}), each row carrying \texttt{\{scene\_id, image\_path, question, answer, kind, object\_names, object\_coords\}}. Splits are scene-disjoint (no scene appears in both files); within a scene, all questions go to the same split.


\paragraph{More examples of generated data}

See Figure \ref{fig:data_example_1}, \ref{fig:data_example_2}, \ref{fig:data_example_3}, \ref{fig:data_example_4}

\begin{figure}[!h]
	\centering
	\includegraphics[width=1\linewidth]{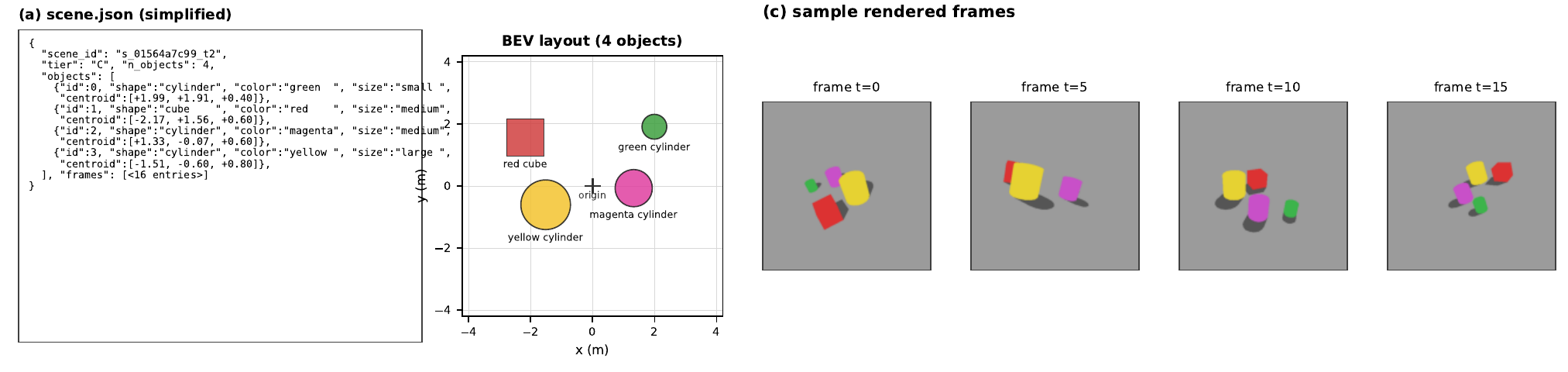}
	\vspace{-0.2in}
	\caption{Generated data, example 1}
	\label{fig:data_example_1}
\end{figure}

\begin{figure}[!h]
	\centering
	\includegraphics[width=1\linewidth]{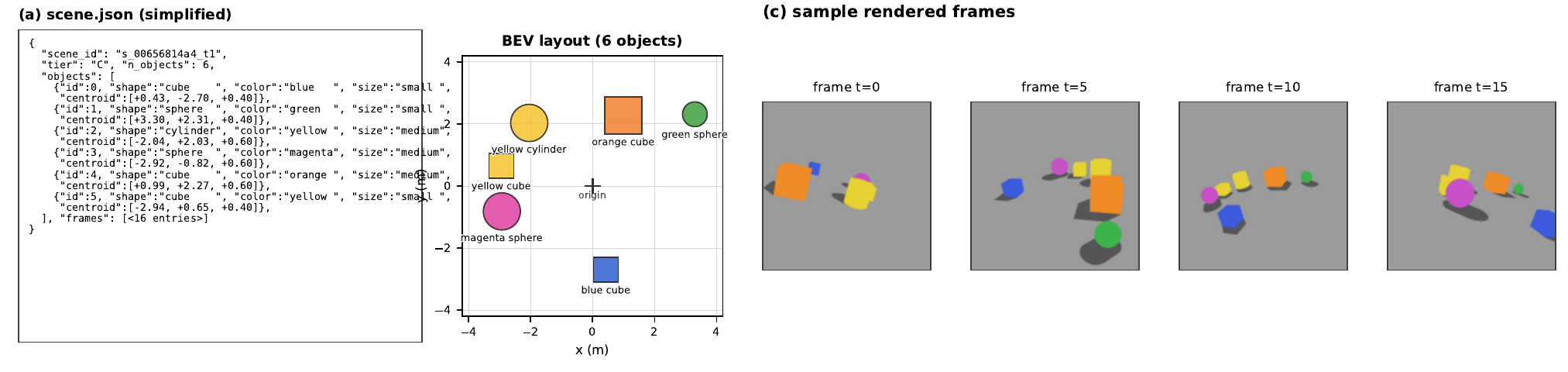}
	\vspace{-0.2in}
	\caption{Generated data, example 2}
	\label{fig:data_example_2}
\end{figure}

\begin{figure}[!h]
	\centering
	\includegraphics[width=1\linewidth]{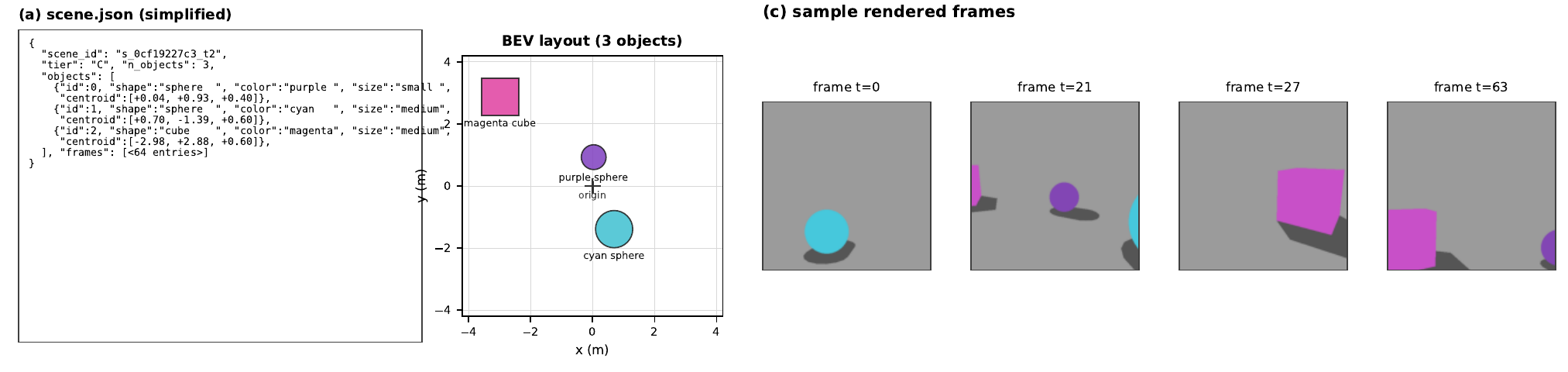}
	\vspace{-0.2in}
	\caption{Generated data, example 3}
	\label{fig:data_example_3}
\end{figure}

\begin{figure}[!h]
	\centering
	\includegraphics[width=1\linewidth]{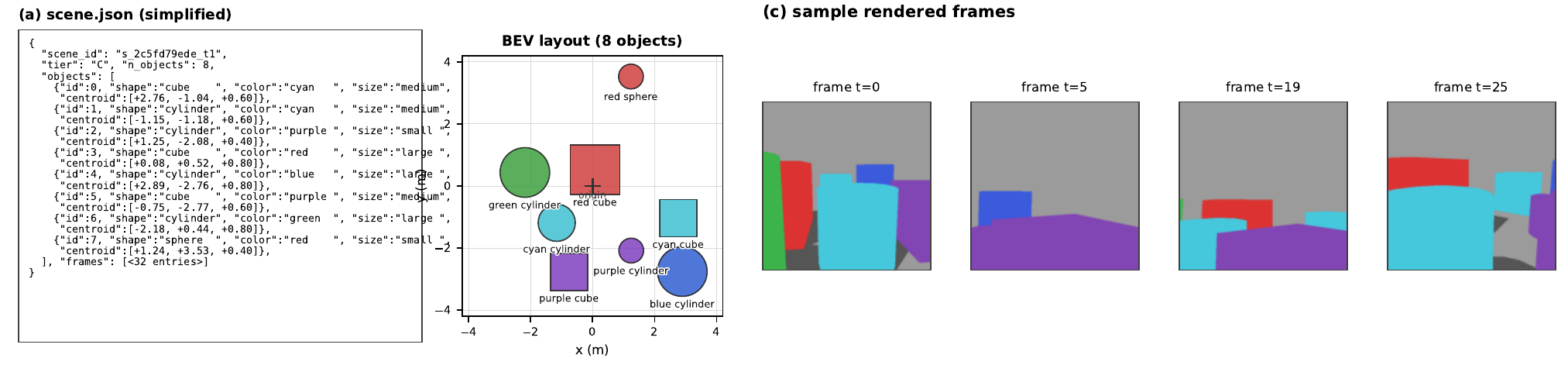}
	\vspace{-0.2in}
	\caption{Generated data, example 4}
	\label{fig:data_example_4}
\end{figure}

\section{Diagnostic experiments}
\label{app:diagnose}

This appendix expands on the diagnostic experiments referenced from \S\ref{sec:diagnose}

\paragraph{Linear probing setup and full results.} For each (model, layer) pair we extract residual-stream activations $h^{(s,o)}_\ell \in \mathbb{R}^d$ at every object-token position over the $1{,}000$-scene Free6DoF probing corpus (Appendix~\ref{app:data_gen}), pooled across all $T=16$ frames per scene with the same mask-driven, coverage-weighted scheme used in \S\ref{sec:extract}. We then fit four independent probes via $5$-fold cross-validation (folds at the scene level, so no scene appears in both train and test): an $8$-way $k$-nearest-neighbor classifier ($k=5$) for object color, a $4$-way $k$-NN classifier for object shape, ridge regressions ($\alpha = 1.0$) on the physical $x$- and $z$-coordinates, and an RSA computation comparing the activation cosine-similarity matrix to the Gaussian-kernel similarity on 3D coordinates (\S\ref{sec:diagnose}, Eqs.~\ref{eq:rsa_mats}--\ref{eq:rsa}). Probe metrics for Qwen2.5-VL-7B at the cognitive-map layer L17 (the layer panels~(a)/(b) of Fig.~\ref{fig:attention_dominance} report on):
\begin{center}
	\small
	\begin{tabular}{lccccc}
		
		\toprule
		target              & color (8-way)        & shape (4-way)         & $x$-coord ($R^2$)   & $z$-coord ($R^2$)   & 3D-distance RSA ($\rho$) \\
		\midrule
		results & $0.51 \pm 0.02$      & $1.00 \pm 0.00$       & $-0.09 \pm 0.01$    & $+0.28 \pm 0.02$    & $+0.01 \pm 0.01$ \\
		chance              & $0.125$              & $0.25$                & $0$ (predict mean)  & $0$ (predict mean)  & $0$ \\
		\bottomrule
	\end{tabular}
\end{center}
The $x$-coordinate $R^2$ is \emph{negative} --- the probe does strictly worse than predicting the population mean --- while color is decoded at $\sim\!4\times$ chance and shape essentially perfectly. The $z$-coordinate is partially recoverable ($R^2 \approx 0.28$), consistent with monocular depth cues being available; we discuss this depth-shortcut concern as a residual confound in \S\ref{sec:diagnose}. Identical patterns reproduce on InternVL3-8B at L18 with the same direction (color/shape decoded near-perfectly, position barely above mean-predictor). Mean$\pm 2\,\mathrm{SE}$ over $n=4$ seeds throughout.


\paragraph{Activation-variance subspace analysis.} Because a linear-probe-accuracy gap could in principle hide a representation that is encoded faithfully but in fewer dimensions, we additionally compare the \emph{activation variance} carried by each feature's optimal low-rank subspace. Concretely, the color subspace is the $7$-dimensional class-mean-difference basis spanning the $8$-color affine hull (i.e., the orthogonal complement of the global mean within the $8$-class centroid subspace); the shape subspace is the $1$-dimensional axis discriminating the $4$ shapes after class-mean centering (rank deficiency from the categorical leaves a $3$-dim residual, but the leading axis carries the bulk of the variance); the position subspace is the $3$-dim basis recovered by ridge regression of 3D coordinates on $H$ and orthogonalized against the color subspace. The fraction of the layer's Frobenius energy $\|H\|_F^2$ captured by each subspace, on Qwen2.5-VL-7B at L17:
\[
v_\mathrm{color\,7d} = 11.86\% \pm 0.24,\quad
v_\mathrm{shape\,1d} = 0.88\% \pm 0.05,\quad
v_\mathrm{pos\,3d} = 0.085\% \pm 0.003.
\]
Color and shape do not just appear in more dimensions; \emph{each} of those dimensions is louder than its position counterpart --- a $\sim 60\times$ gap per dimension and a $\sim 140\times$ aggregate gap between color and position variance (Fig.~\ref{fig:attention_dominance}b). The qualitative pattern replicates on InternVL3-8B at L18 (color $1.54\%$, shape $1.08\%$, position $0.066\%$ --- a $\sim 23\times$ aggregate gap with the matched small-basis estimator).

\paragraph{Counterfactual color/position swap.}
The counterfactual experiment summarized in \S\ref{sec:diagnose} perturbs the input scene rather than the model's activations. We sample $50$ base scenes from the Free6DoF corpus that have $\ge\!4$ distinct (color, shape) identities so the swap permutations are non-trivial. From each base scene we generate three rendered variants:
\begin{itemize}[leftmargin=*, itemsep=2pt, topsep=2pt]
	\item \emph{Original}: the scene rendered as-is.
	\item \emph{Color-swap}: the per-object \emph{colors} are permuted by a random derangement (no object keeps its original color), but every $3$D position, shape, size, and camera trajectory is held fixed; the scene is re-rendered.
	\item \emph{Position-swap}: the per-object \emph{positions} are permuted by a random derangement, leaving every color, shape, size, and camera trajectory fixed; re-rendered.
\end{itemize}
Per scene we generate $\sim\!9$ spatial questions across two kinds: \emph{distance-order} (``Which is closer to the magenta cube: the purple sphere or the cyan cube?'') and \emph{relative-position} (``Is the yellow cylinder to the left of the red cylinder?''). Each question is asked against all three variants of the scene, yielding $1{,}290$ QA pairs per random seed. The QA jsonl and pre-rendered variant frames live at \texttt{data/counterfactuals/}. We measure the per-(scene, question) \emph{answer-flip rate}
\[
\mathrm{flip}_v \;=\; \Pr\!\bigl(\hat{y}_v \neq \hat{y}_{\rm orig}\bigr), \qquad v \in \{\text{color-swap}, \text{position-swap}\},
\]
under greedy decoding at \texttt{bf16}, averaged over $n=4$ seeds. Per-question-kind breakdown on Qwen2.5-VL-7B (lam0 baseline):
\begin{center}
	\small
	\begin{tabular}{lccc}
		\toprule
		question kind & color-swap flip & position-swap flip & expected (position-driven head) \\
		\midrule
		distance-order   & $\mathbf{41.0\%}$  & $48.6\%$  & color-swap $\to 0\%$ \\
		relative-position & $3.5\%$ & $7.3\%$ & both $\to 0\%$ \\
		\bottomrule
	\end{tabular}
\end{center}
On distance-order questions the color-swap flip rate is the \emph{same order of magnitude} as the position-swap flip rate ($41.0\%$ vs.\ $48.6\%$), even though color-swap leaves every object's $3$D position untouched. The expected behavior of a head that grounds ``the magenta cube'' through internal 3D position would be a near-zero color-swap flip rate. On relative-position questions, the purely semantic color-swap alters the model's prediction $3.5\%$ of the time, which is expected to be 0. This result provides behavioral evidence for the representational imbalance identified earlier in \S\ref{sec:diagnose}. 


\section{Causal verification by activation steering}
\label{app:steering}

This appendix expands on the ``Causal verification'' paragraph in \S\ref{sec:extract} (Fig.~\ref{fig:casual_verification}). The motivation: linear probing and the variance-subspace analysis (Appendix~\ref{app:diagnose}) only tell us what is \emph{decodable} from the residual stream; they cannot rule out the possibility that the LM head simply ignores the probe's subspace and reads from a different code. To close that gap we ran an interventional experiment in which we directly inject a controlled vector along the probe's principal axis into the residual stream and measure whether the model's internal 3D readout shifts as predicted.

\paragraph{Setup.}
\begin{center}
	\begin{tabular}{ll}
		\toprule
		backbone                  & Qwen2.5-VL-7B-Instruct \\
		data                      & Tier-C Free6DoF, $5$ scenes $\times$ $3$ objects \\
		steered layer             & $\ell = 12$ (chosen by held-out probe $R^2$) \\
		targeted axis             & world-frame $x$ (per-scene-normalized) \\
		intervention magnitudes   & $\alpha \in \{-0.30, -0.15, +0.15, +0.30\}$ (norm-$x$ units) \\
		trials per (direction, $\alpha$) & $15$ \\
		control direction         & length-matched null-space vector $v_\perp$ \\
		\bottomrule
	\end{tabular}
\end{center}

\paragraph{Probe-axis and control construction.} We fit a Ridge probe at layer $\ell\!=\!12$ that maps the pooled object-token residual-stream activation to per-scene-normalized 3D coordinates: $\hat{x} = W h$ with $W \in \mathbb{R}^{3 \times d}$ and $\ell_2$ regularization $\alpha_{\rm ridge}\!=\!1.0$. The steering direction $v_{\rm axis} \in \mathbb{R}^d$ is the row of $W^\dagger$ corresponding to the world-frame $x$ axis (so injecting $\alpha\!\cdot\!v_{\rm axis}$ shifts the probe's $\hat{x}$ readout by $\alpha$ to first order, while leaving $\hat{y}$ and $\hat{z}$ unchanged). The length-matched control $v_\perp$ is the projection of a random Gaussian vector onto the null space of $W$ (so $W v_\perp = 0$ to numerical precision), rescaled to $\|v_\perp\|\!=\!\|v_{\rm axis}\|$. By construction $v_\perp$ encodes ``a direction in the residual stream of the same magnitude as $v_{\rm axis}$ that the probe cannot read.''

\paragraph{Intervention.} For a chosen object in a chosen scene, we register a forward-pre-hook on layer $12$ that adds $\alpha\!\cdot\!v$ to that object's residual at every visual-token position belonging to that object (mask-driven, same indices used for object-token pooling in \S\ref{sec:extract}); other tokens are untouched. We then re-run the model and re-read the probe's 3D-coordinate prediction $(\hat{x}, \hat{y}, \hat{z})$ for the same object. The diagnostic statistic is $\Delta\hat{x} = \hat{x}_{\rm intervention} - \hat{x}_{\rm baseline}$ (and analogously for $\hat{y}, \hat{z}$).

\paragraph{Targeted shift along $v_{\rm axis}$ vs.~$v_\perp$.} Effect of steering on the probe's $x$-axis readout (mean $\pm$ SEM, $n\!=\!15$ per cell):
\begin{center}
	\begin{tabular}{cccccc}
		\toprule
		$\alpha$ & $v_{\rm axis}$ & $v_\perp$ & ratio $|axis|/|perp|$ & Welch $t$ & $p$ \\
		\midrule
		$-0.30$ & $\mathbf{-0.1020 \pm 0.027}$ & $+0.0201 \pm 0.035$ & $4.9\times$ & $-5.58$ & $6.8\!\times\!10^{-6}$ \\
		$-0.15$ & $\mathbf{-0.0558 \pm 0.030}$ & $+0.0045 \pm 0.031$ & $12.9\times$ & --- & --- \\
		$+0.15$ & $\mathbf{+0.0706 \pm 0.033}$ & $+0.0087 \pm 0.025$ & $8.4\times$ & --- & --- \\
		$+0.30$ & $\mathbf{+0.0908 \pm 0.021}$ & $+0.001 \pm 0.030$ & $\mathbf{82.2\times}$ & $+5.04$ & $3.3\!\times\!10^{-5}$ \\
		\bottomrule
	\end{tabular}
\end{center}
The probe-readout shift along $v_{\rm axis}$ is monotonic in $\alpha$ and aligned with its sign; along $v_\perp$ it is at the noise floor across all four magnitudes. Welch $t$-tests at $\alpha\!=\!\pm 0.30$ separate the two arms at $p < 10^{-4}$.

\paragraph{Sign-correctness.} Counting trials whose readout shift has the same sign as the injected $\alpha$:
\begin{center}
	\begin{tabular}{lcc}
		\toprule
		direction & trials matching $\mathrm{sign}(\alpha)$ & rate \\
		\midrule
		$v_{\rm axis}$ & $54 / 60$ & $\mathbf{90.0\%}$ \\
		$v_\perp$      & $29 / 60$ & $48.3\%$ (chance level) \\
		\bottomrule
	\end{tabular}
\end{center}
The $90\%$-vs-chance gap rules out the alternative that the readout shift is a generic effect of perturbing the residual stream with any vector of the right magnitude.

\paragraph{Interpretation.} The probe's row of $W^\dagger$ is, by construction, the direction along which the trained Ridge readout responds linearly to the targeted coordinate. The $90\%$ sign-correctness, $\sim\!82\times$ magnitude advantage at $\alpha\!=\!\pm 0.30$, and clean axis selectivity together demonstrate that this same direction is causally read by the layer-$12$ residual stream's downstream computation: pushing the activation along $v_{\rm axis}$ moves the model's internal $x$-readout in the predicted direction with the predicted sign, and pushing along a length-matched null-space direction does not. This is the strongest causal evidence in our diagnostic suite that the linearly-extracted spatial subspace is what the model actually uses, not an artifact of the probe.

\section{Visualization examples}
\label{app:v_eg}

In Figures \ref{fig:circular1}, \ref{fig:circular2}, \ref{fig:grid1}, \ref{fig:grid2}, \ref{fig:random1}, \ref{fig:random2}, and \ref{fig:random3}, we show more examples of visualization for extracted spatial feature in different layers, similar to those in Figure \ref{fig:pca_visualized}.

\begin{figure}[!h]
	\centering
	\vspace{-0.15in}
	\includegraphics[width=1\linewidth]{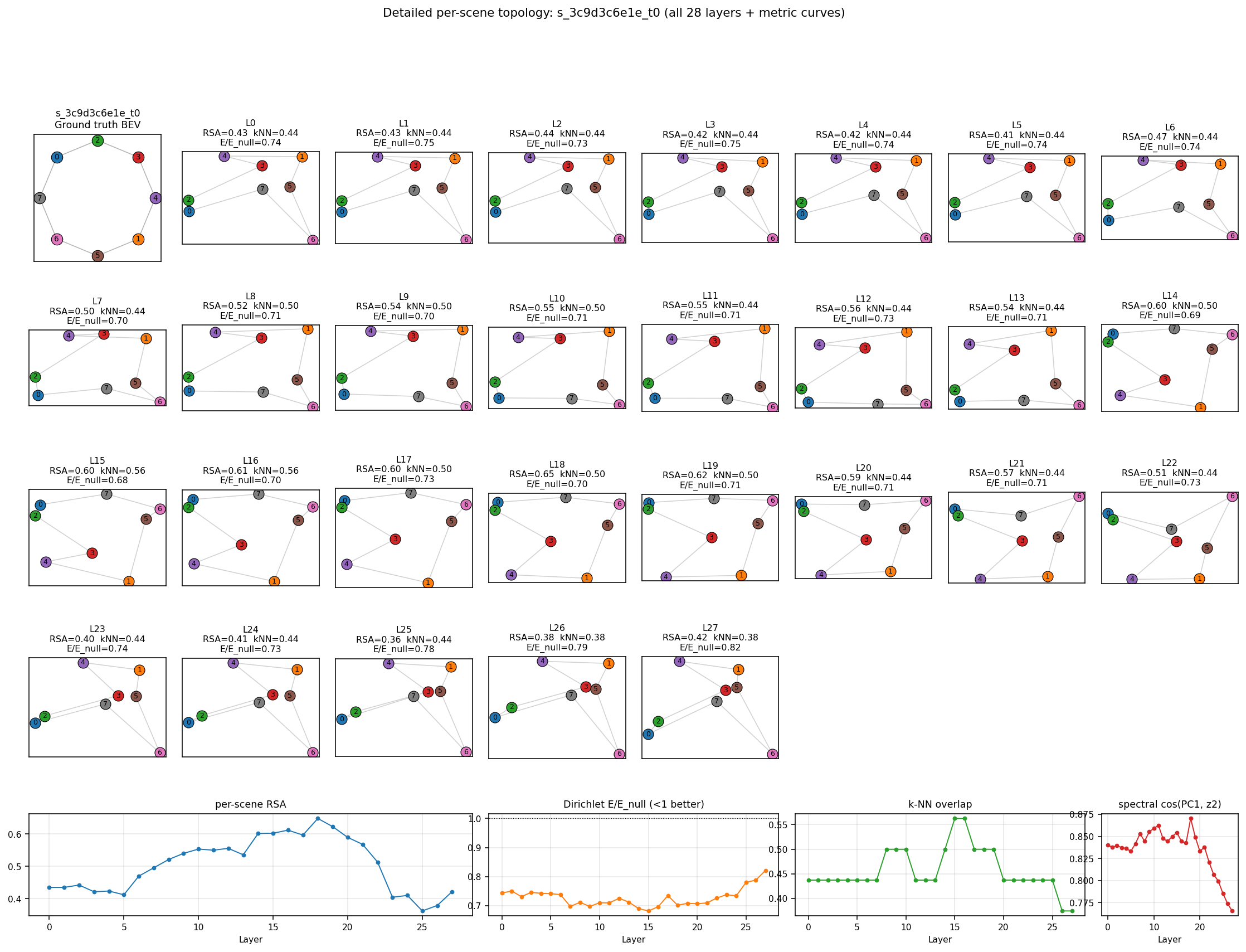}
	\vspace{-0.2in}
	\caption{Circular Layout, Example 1}
	\vspace{-0.25in}
	\label{fig:circular1}
\end{figure}

\begin{figure}[!h]
	\centering
	\vspace{-0.15in}
	\includegraphics[width=1\linewidth]{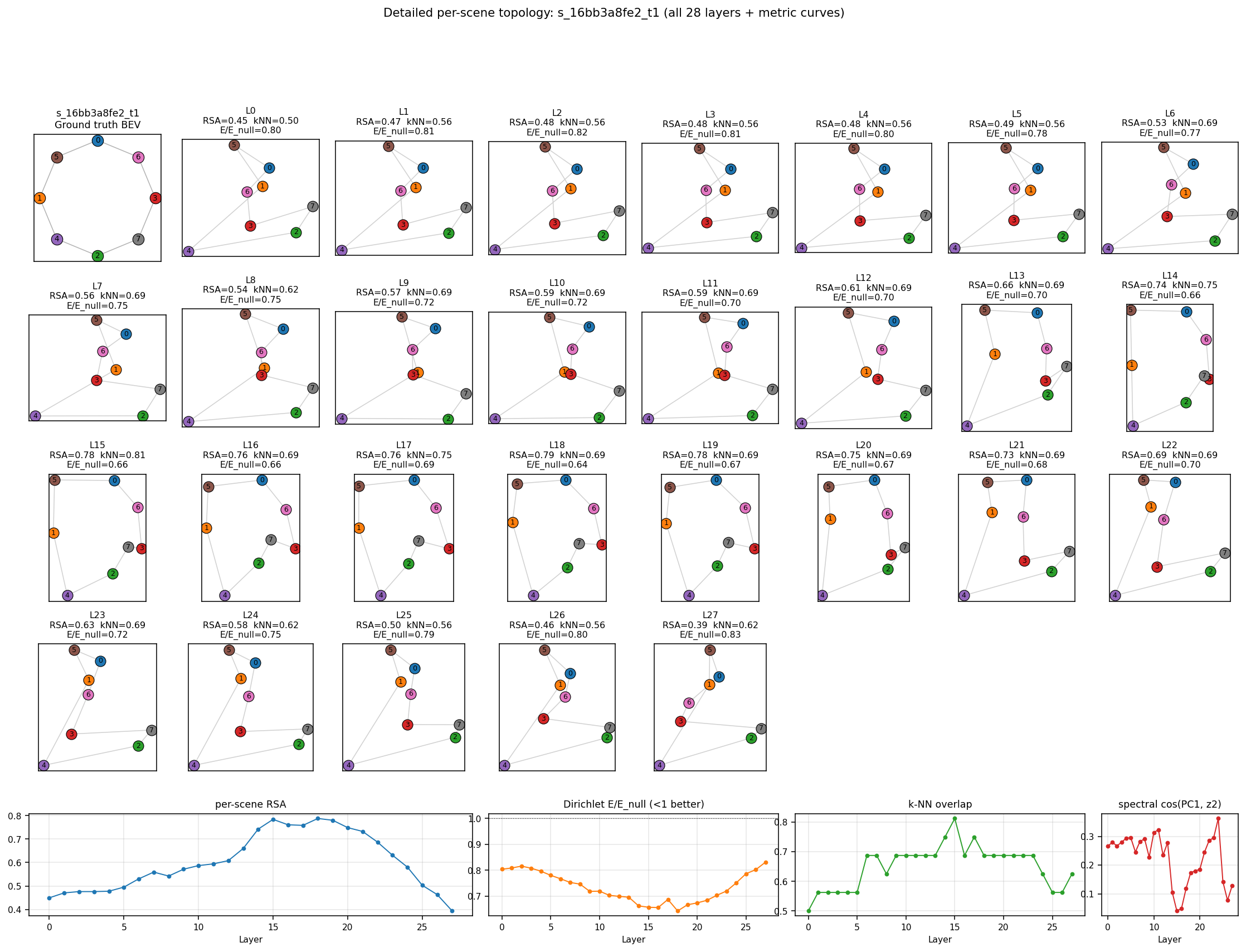}
	\vspace{-0.2in}
	\caption{Circular Layout, Example 2}
	\vspace{-0.25in}
	\label{fig:circular2}
\end{figure}

\begin{figure}[!h]
	\centering
	\vspace{-0.15in}
	\includegraphics[width=1\linewidth]{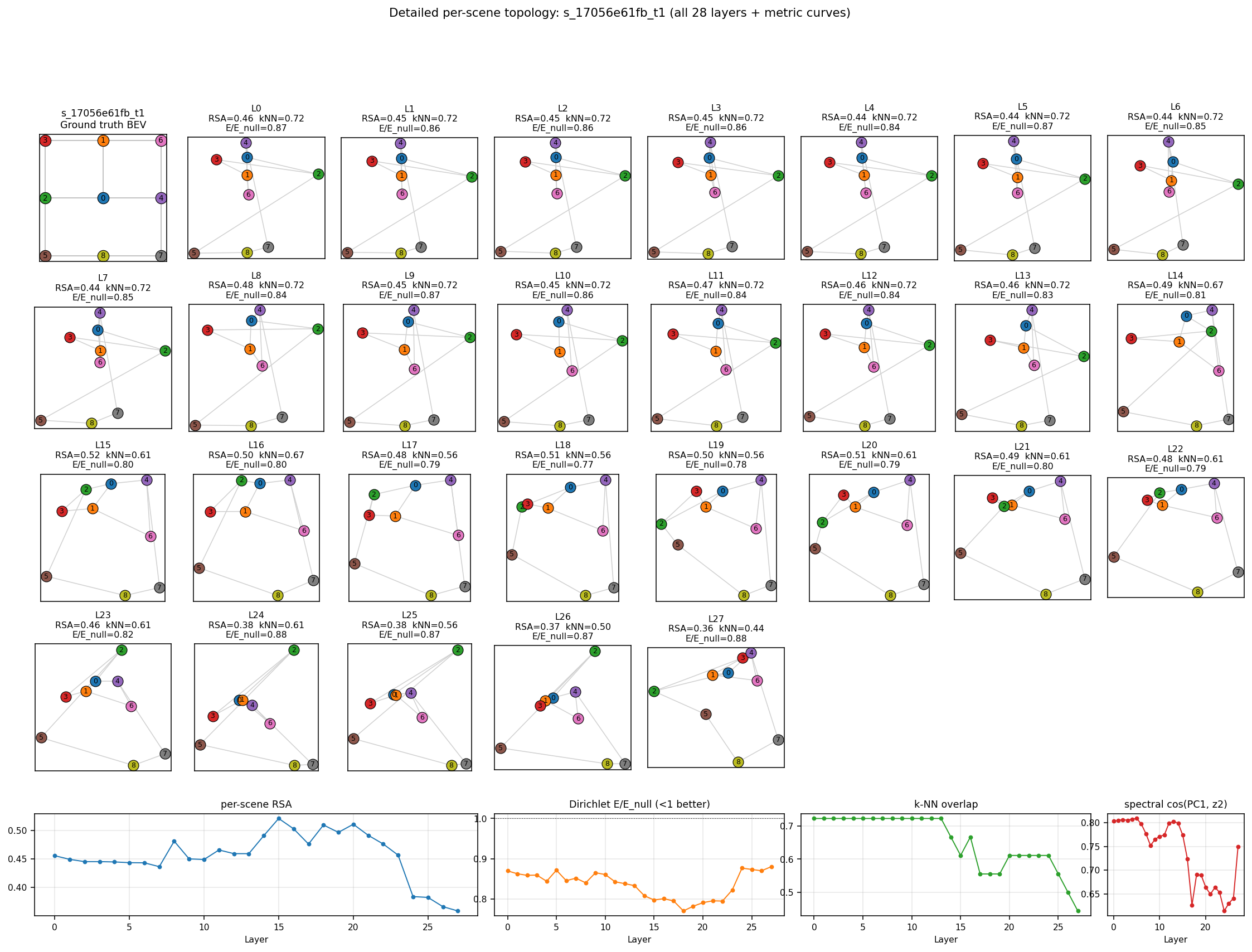}
	\vspace{-0.2in}
	\caption{Grid Layout, Example 1}
	\vspace{-0.25in}
	\label{fig:grid1}
\end{figure}

\begin{figure}[!h]
	\centering
	\vspace{-0.15in}
	\includegraphics[width=1\linewidth]{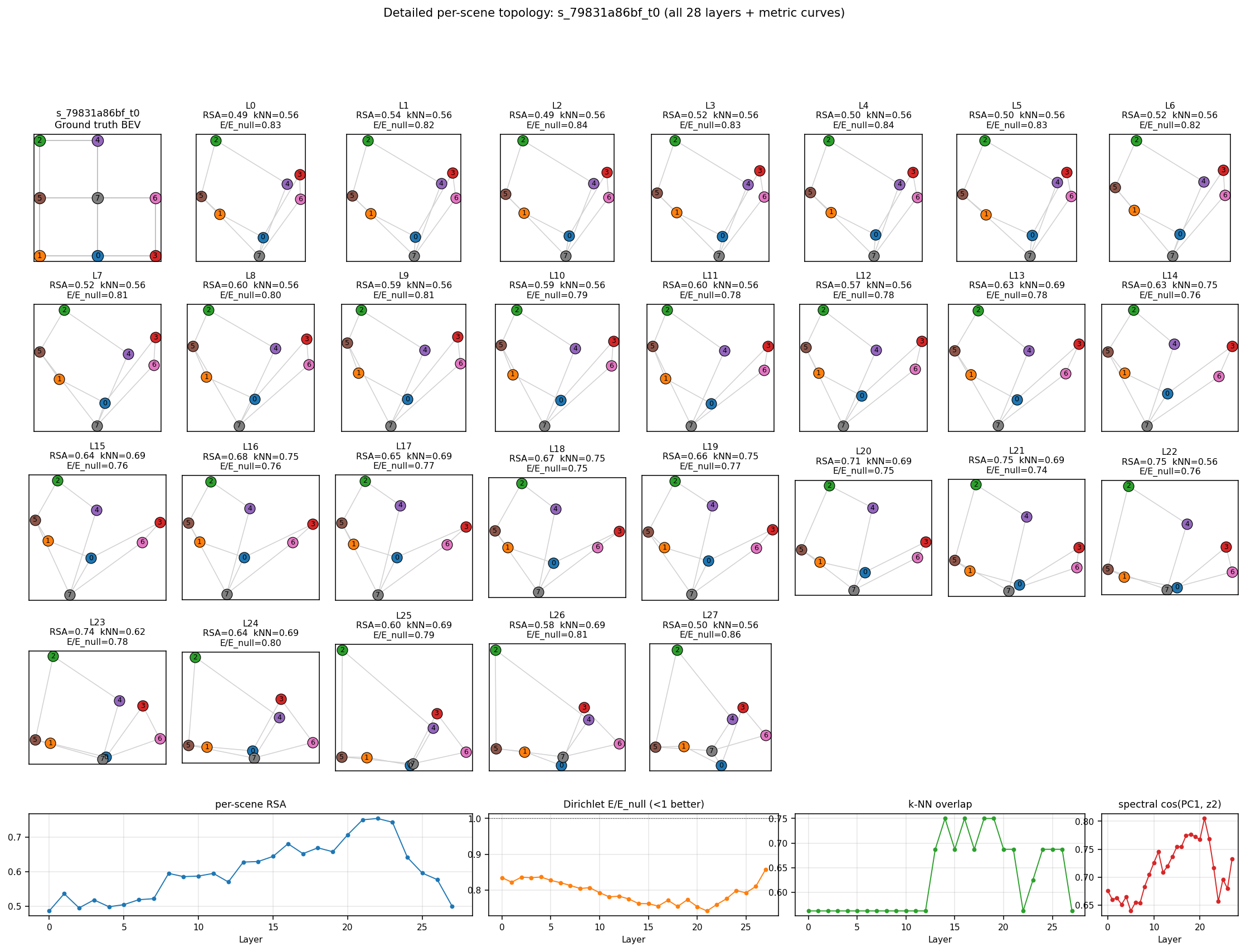}
	\vspace{-0.2in}
	\caption{Grid Layout, Example 2}
	\vspace{-0.25in}
	\label{fig:grid2}
\end{figure}

\begin{figure}[!h]
	\centering
	\vspace{-0.15in}
	\includegraphics[width=1\linewidth]{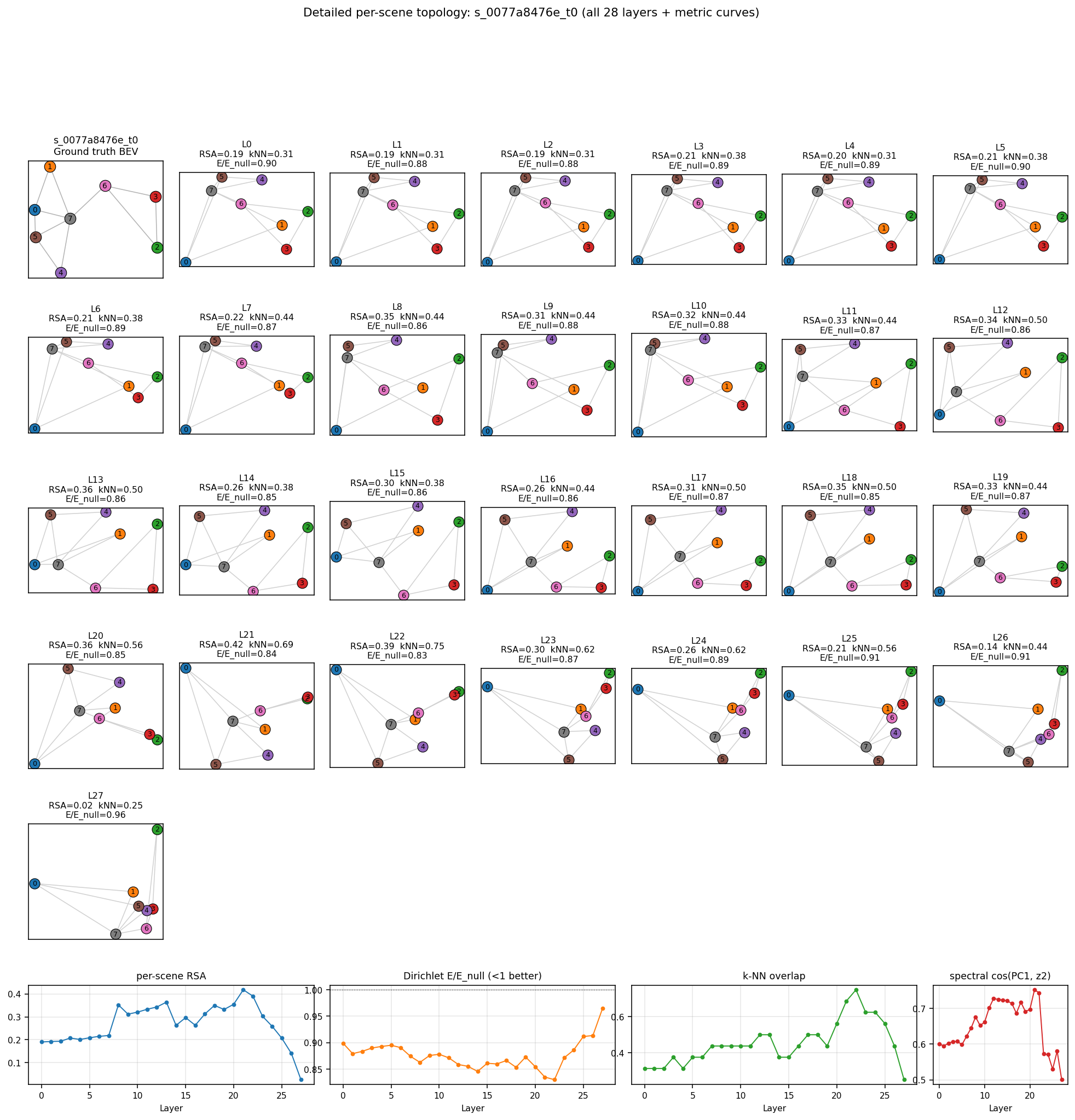}
	\vspace{-0.2in}
	\caption{Random Layout, Example 1}
	\vspace{-0.25in}
	\label{fig:random1}
\end{figure}

\begin{figure}[!h]
	\centering
	\vspace{-0.15in}
	\includegraphics[width=1\linewidth]{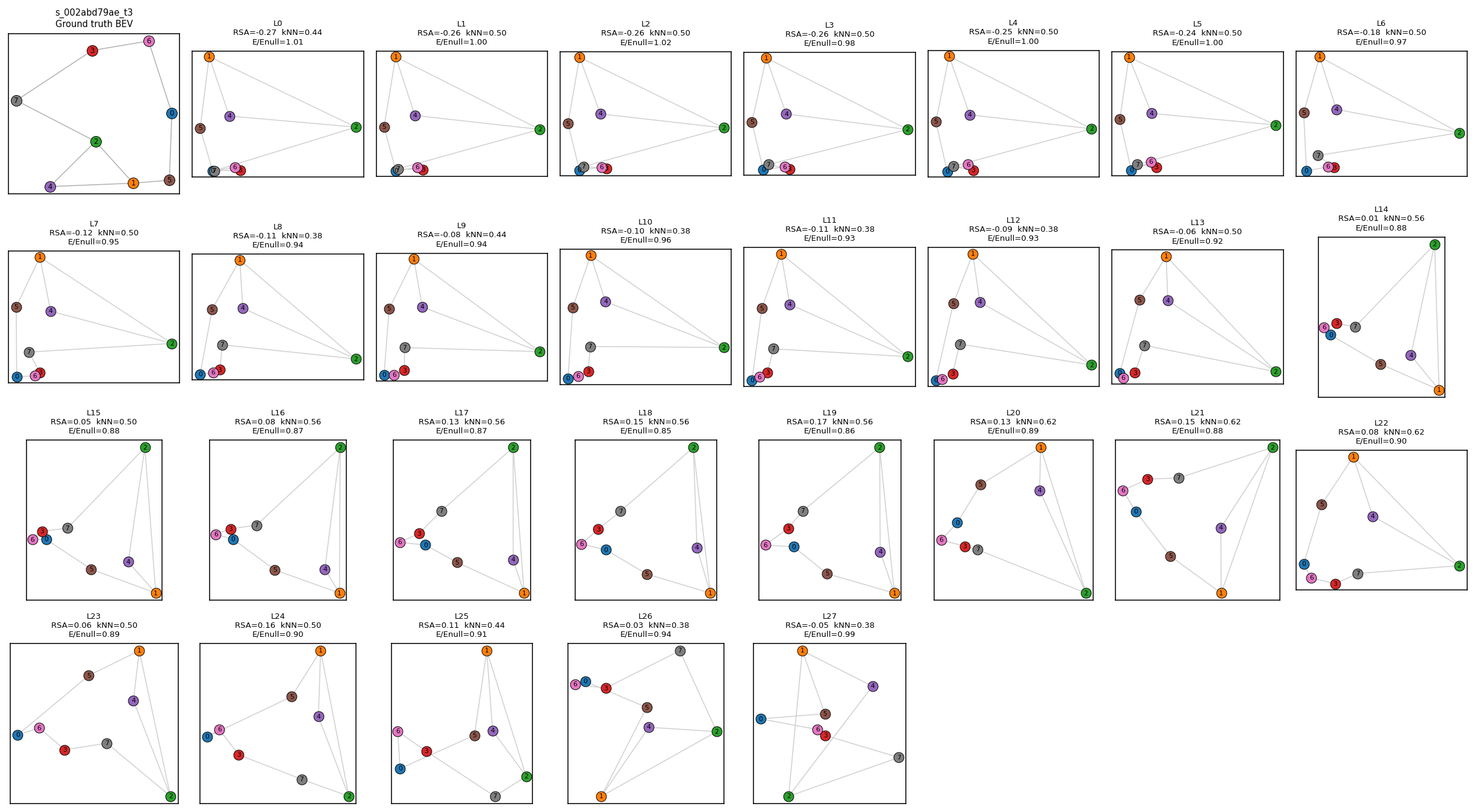}
	\vspace{-0.2in}
	\caption{Random Layout, Example 2}
	\vspace{-0.25in}
	\label{fig:random2}
\end{figure}

\begin{figure}[!h]
	\centering
	\vspace{-0.15in}
	\includegraphics[width=1\linewidth]{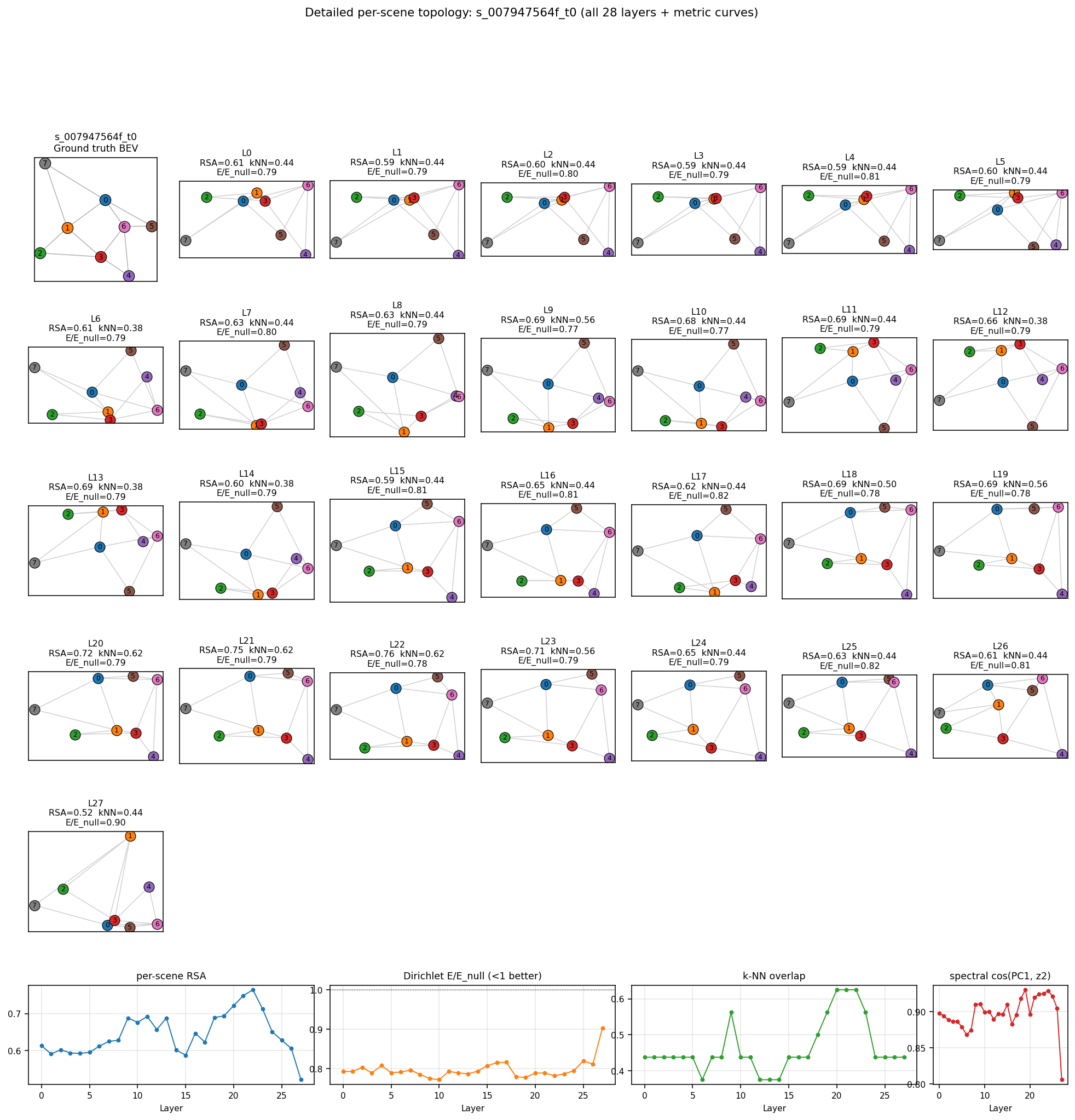}
	\vspace{-0.2in}
	\caption{Random Layout, Example 3}
	\vspace{-0.25in}
	\label{fig:random3}
\end{figure}

\clearpage

\section{Choice of identity-attribute basis rank \texorpdfstring{$k$}{k} and per-task results}
\label{app:fulltable}

This appendix supports two references in the main text: (i) the choice of the basis rank $k$ in Eq.~\ref{eq:basis} for the identity-attribute subspace $W_\ell$, and (ii) the extended per-task numbers underlying Tables~\ref{tab:overall},~\ref{tab:qwen_vsi}, and~\ref{tab:intern_vsi}.

\subsection{Choice of \texorpdfstring{$k$}{k} for the identity-attribute basis}
\label{app:choice_k}

Recall that $\bar H_\ell \in \mathbb{R}^{K \times d}$ stacks the per-identity prototypes (\S\ref{sec:extract}) over $K$ unique (color, shape) identities in the synthetic inventory --- $K = 8 \times 3 = 24$ on the Free6DoF corpus. The basis $W_\ell \in \mathbb{R}^{d \times k}$ is then the top-$k$ left singular vectors of $\bar H_\ell^\top$, and the projector $P_\perp = I_d - W_\ell W_\ell^\top$ residualizes against the span of identity directions.

\paragraph{Singular-value spectrum.} Because color (8-way) and shape (3-way) are sampled independently in the data generator (Appendix~\ref{app:data_gen}), the prototype cloud has full effective rank up to $K-1$ (one direction is consumed by the global mean). Empirically the spectrum of $\bar H_\ell^\top$ at the cognitive-map layer is well-approximated by a power-law decay with no sharp gap: the first 8 components carry $\sim 60\%$ of the prototype variance (the dominant color axes), the next 3 carry $\sim 20\%$ (shape), and the remaining $\sim 13$ carry the residual $\sim 20\%$ (color$\times$shape interactions). There is no clean ``elbow'' before $k = K - 1$, so we keep all non-degenerate components.

\paragraph{Default choice.} We set $k = K - 1 = 23$ on Free6DoF. This (i) spans the full identity prototype cloud up to global translation, (ii) does not require a hyperparameter search, and (iii) coincides with the rank that is naturally produced by orthonormalizing the multinomial logistic-regression directions for color and shape via thin QR after dropping the rank-deficient class direction per categorical (cf.~Appendix~\ref{app:dirichlet_impl}; effective rank $\le 7+2 = 9$ when only the discriminating directions are kept, and $\le K-1 = 23$ when the entire prototype cloud is used).

\paragraph{Sensitivity.} We measured downstream RSA $\rho_{\rm topo}$ at the cognitive-map layer of Qwen2.5-VL-7B as a function of $k$:
\begin{center}
	\begin{tabular}{lcccccc}
		\toprule
		$k$ & 4 & 8 & 11 & 16 & 23 & 32 \\
		\midrule
		$\rho_{\rm topo}$ at L17 ($\uparrow$) & $0.31$ & $0.48$ & $0.55$ & $0.59$ & $\mathbf{0.61}$ & $0.60$ \\
		Color-probe acc. on $\widetilde h$ ($\downarrow$, chance $= 1/8$) & $0.42$ & $0.21$ & $0.16$ & $0.14$ & $\mathbf{0.13}$ & $0.13$ \\
		\bottomrule
	\end{tabular}
\end{center}
RSA increases sharply up to $k\!\approx\!11$, then plateaus by $k = 23$; pushing $k$ past the prototype rank (e.g.\ $k = 32$) projects out genuinely random directions and degrades both metrics slightly. Since the curves are flat over $k \in [16, 32]$, the residualization is robust to mis-estimating $k$ within this band; we report $k = 23$ in all main-paper experiments.

\section{Full Proof of Theorem~\ref{thm:eigenmap} (Dirichlet Minimization)}
\label{app:proofs}

This appendix provides the rigorous algebraic derivation for Theorem~\ref{thm:eigenmap}, which states that minimizing the Dirichlet energy under singular value constraints yields the Laplacian eigenmaps as principal components.

\subsection{Background: The Weighted Ky Fan Inequality}
We first establish a weighted form of Ky Fan's trace minimum theorem.

\begin{lemma}[Ky Fan, weighted form]
	\label{lem:kyfan}
	Let $L \in \mathbb{R}^{m \times m}$ be a symmetric matrix with eigenvalues $\lambda_1 \le \lambda_2 \le \dots \le \lambda_m$ and corresponding orthonormal eigenvectors $z^{(1)}, \dots, z^{(m)}$. Let $w_1 > w_2 > \dots > w_s > 0$ be a strictly decreasing positive sequence (with $s \le m$). Then,
	\begin{equation}
		\min_{\substack{u_1, \dots, u_s \in \mathbb{R}^m \\ \langle u_i, u_j \rangle = \delta_{ij}}} \sum_{k=1}^s w_k \langle u_k, L u_k \rangle \;=\; \sum_{k=1}^s w_k \lambda_k,
		\label{eq:kyfan}
	\end{equation}
	and the unique (up to sign) minimizer is $u_k = z^{(k)}$ for each $k$.
\end{lemma}

\begin{proof}
	Substituting $u_k = z^{(k)}$ trivially yields $\langle z^{(k)}, L z^{(k)} \rangle = \lambda_k$, attaining the right-hand side. The substantive part is showing that any orthonormal set $\{u_k\}$ achieves at least the right-hand side. 
	We use summation by parts (Abel summation). Define cumulative weights $W_s \triangleq w_s$ and consecutive differences $\Delta_k \triangleq w_k - w_{k+1}$ for $k = 1, \dots, s-1$. The strictly decreasing assumption guarantees $\Delta_k > 0$ for all $k$, and $W_s > 0$. We telescope the weights: $w_k = W_s + \sum_{j=k}^{s-1} \Delta_j$. Substituting this into the sum and swapping the order of summation yields:
	\begin{align*}
		\sum_{k=1}^s w_k \langle u_k, L u_k \rangle 
		&= \sum_{k=1}^s \left(W_s + \sum_{j=k}^{s-1} \Delta_j\right) \langle u_k, L u_k \rangle \\
		&= W_s \sum_{k=1}^s \langle u_k, L u_k \rangle + \sum_{j=1}^{s-1} \Delta_j \sum_{k=1}^j \langle u_k, L u_k \rangle.
	\end{align*}
	By the standard Ky Fan theorem for uniform weights (or the Poincaré separation theorem), any orthonormal prefix $\{u_1, \dots, u_j\}$ satisfies $\sum_{k=1}^j \langle u_k, L u_k \rangle \ge \sum_{k=1}^j \lambda_k$. Substituting this lower bound into every partial sum, we get:
	\begin{align*}
		\sum_{k=1}^s w_k \langle u_k, L u_k \rangle 
		&\ge W_s \sum_{k=1}^s \lambda_k + \sum_{j=1}^{s-1} \Delta_j \sum_{k=1}^j \lambda_k \\
		&= \sum_{k=1}^s \left(W_s + \sum_{j=k}^{s-1} \Delta_j\right) \lambda_k \;=\; \sum_{k=1}^s w_k \lambda_k.
	\end{align*}
	The uniqueness up to sign follows strictly from the strict monotonicity of the weights, which forces equality on every individual partial sum, uniquely binding $u_k$ to the eigenspace of $z^{(k)}$.
\end{proof}

\subsection{Proof of Theorem~\ref{thm:eigenmap}}

The proof proceeds in four explicit steps: decomposing the energy via SVD, saturating the constraints, identifying the optimal singular vectors via Lemma~\ref{lem:kyfan}, and addressing the mean-centering.

\paragraph{Step 1: Decompose the energy along the SVD.}
We wish to minimize $\mathcal{E}_X(H) = \text{tr}(H^\top L H)$. Let the (thin) SVD of $H$ be $H = U \Sigma V^\top$, where $U \in \mathbb{R}^{m \times r}$ and $V \in \mathbb{R}^{d \times r}$ have orthonormal columns, $\Sigma = \text{diag}(\sigma_1, \dots, \sigma_r)$, and $r = \min(m,d)$. Substituting this into the trace:
\[
\text{tr}(H^\top L H) = \text{tr}((V \Sigma U^\top) L (U \Sigma V^\top)).
\]
By the cyclic property of the trace, $\text{tr}(V \Sigma U^\top L U \Sigma V^\top) = \text{tr}(\Sigma U^\top L U \Sigma V^\top V)$. Since $V^\top V = I_r$, we obtain $\text{tr}(\Sigma^2 U^\top L U)$. Because $\Sigma^2$ is diagonal, this is precisely a weighted sum over the columns $u_k$ of $U$:
\begin{equation}
	\mathcal{E}_X(H) = \sum_{k=1}^r \sigma_k^2 \langle u_k, L u_k \rangle.
	\label{eq:energy_svd}
\end{equation}

\paragraph{Step 2: The lower-bound constraints saturate.}
Equation~\ref{eq:energy_svd} is a sum of non-negative terms, because $L$ is positive semi-definite (implying $\langle u_k, L u_k \rangle \ge 0$). For each term, the singular value $\sigma_k$ operates solely as a non-negative coefficient. Within the feasible region dictated by our constraints $\sigma_k \ge \epsilon_k > 0$ for $k \in [s]$ (and $\sigma_k \ge 0$ for $k > s$), minimizing the sum with respect to the singular values trivially pushes them to their absolute lowest bounds. Thus, at the minimum $H^*$, we have exactly $\sigma_k(H^*) = \epsilon_k$ for $k \le s$, and $\sigma_k(H^*) = 0$ for $k > s$. The problem therefore reduces to optimizing the left-singular vectors $u_k$:
\begin{equation}
	\min_{\substack{u_1, \dots, u_s \in \mathbb{R}^m \\ u_i \perp u_j}} \sum_{k=1}^s \epsilon_k^2 \langle u_k, L u_k \rangle.
	\label{eq:reduced_opt}
\end{equation}

\paragraph{Step 3: Apply Ky Fan.}
Equation~\ref{eq:reduced_opt} perfectly matches the hypothesis of Lemma~\ref{lem:kyfan}, where our symmetric matrix is the graph Laplacian $L$, our weights are $w_k = \epsilon_k^2$, and our orthonormal vectors are the left singular vectors $u_k$. The assumption $\epsilon_1 > \dots > \epsilon_s > 0$ ensures the weights are strictly decreasing. Therefore, by Lemma~\ref{lem:kyfan}, the unique (up to sign) minimizers are the eigenvectors of $L$ corresponding to its smallest eigenvalues:
\[
u_k(H^*) = z^{(k)} \quad \text{for } k=1, \dots, s.
\]
This proves part (a) of the theorem. Note that the right-singular vectors $V$ drop out of the objective entirely, giving the "uniqueness up to right-orthogonal transformation" claim.

\paragraph{Step 4: Mean-centering and Principal Components.}
We now evaluate the PCA of the mean-centered optimal representation $H_c^* = P_{J^\perp} H^*$, where $P_{J^\perp} = I - \frac{1}{m}\mathbf{1}\mathbf{1}^\top$ is the mean-centering orthogonal projector.
For the graph Laplacian $L=D-W$, the smallest eigenvalue is always $\lambda_1=0$, and its corresponding eigenvector is the uniform vector $z^{(1)} = \mathbf{1}/\sqrt{m}$. 
Because $P_{J^\perp}$ is the projector onto the orthogonal complement of $\mathbf{1}$, we have $P_{J^\perp} z^{(1)} = 0$. 
For $k \ge 2$, the eigenvectors of the symmetric matrix $L$ are mutually orthogonal, so $z^{(k)} \perp \mathbf{1}$, meaning $P_{J^\perp} z^{(k)} = z^{(k)}$. 

Applying this to the SVD of $H^*$:
\[
H_c^* = P_{J^\perp} (U \Sigma V^\top) = (P_{J^\perp} U) \Sigma V^\top.
\]
The columns of $P_{J^\perp} U$ are $P_{J^\perp} u_k$. For $k=1$, $P_{J^\perp} z^{(1)} = 0$. For $k \ge 2$, $P_{J^\perp} z^{(k)} = z^{(k)}$. Therefore, the non-zero singular components of $H_c^*$ start directly at $z^{(2)}, z^{(3)}, \dots, z^{(s)}$. Equivalently, the $k$-th left singular vector (the $k$-th principal component direction) of $H_c^*$ is exactly $z^{(k+1)}$, proving part (b).

\section{Proofs of Training-Time Guarantees (Theorems 2, 3, 4, 5)}
\label{app:training_theorems}
We use the same notation as Appendix~\ref{app:proofs}: a fixed scene with $m$ object-tokens, residual-stream representation $H \in \mathbb{R}^{m \times d}$, scene 3D coordinates $X \in \mathbb{R}^{m \times 3}$, Gaussian-kernel adjacency $W_{ij} = \exp(-\|x_i - x_j\|^2/2\tau^2)$ for $i \neq j$, degree matrix $D_{ii} = \sum_j W_{ij}$, unnormalized Laplacian $L = D - W$ with eigenpairs $(\lambda_k, z^{(k)})_{k=1}^m$, $\lambda_1 = 0 \le \lambda_2 \le \cdots$, and Dirichlet energy $\mathcal{E}_X(H) = \mathrm{tr}(H^\top L H)$. Throughout, $H^* = U^*\Sigma^* (V^*)^\top$ denotes the Theorem~\ref{thm:eigenmap} minimizer of $\mathcal{E}_X$ subject to $\sigma_k(H) \ge \epsilon_k$ for $k \in [s]$, and $V^*_{:,1:3}$ its top-3 right singular vectors. Note that, by Theorem~\ref{thm:eigenmap}(b), the columns of the mean-centered $U^*$ at indices $1$--$3$ are exactly the Laplacian eigenvectors $z^{(2)}, z^{(3)}, z^{(4)}$.

\subsection{Proof of Theorem~\ref{thm:limit} (continuous limit)}
\label{app:proof_limit}

The proof has three ingredients drawn from spectral graph theory and manifold learning, plus a final identification step on the Euclidean cube. We sketch the argument; full proofs of the operator-convergence results below are standard and can be found in the cited references.

\paragraph{Step 1 (pointwise operator limit, Belkin--Niyogi 2003).}
Let $\widetilde{L}^{(\tau)}_m$ denote the degree-rescaled empirical Laplacian
\[
\widetilde{L}^{(\tau)}_m f(x) \;:=\; \frac{1}{\tau^{d_\Omega+2}}\!\left[\, f(x)\cdot\tfrac{1}{m}\!\sum_{j=1}^m \kappa_\tau(x, x_j) \;-\; \tfrac{1}{m}\!\sum_{j=1}^m \kappa_\tau(x, x_j)\, f(x_j) \,\right],
\]
the function-space lift of the matrix $\frac{1}{m\tau^{d_\Omega+2}}\,L^{(\tau)}$, where $d_\Omega = 3$ is the intrinsic dimension of $\Omega$. Belkin and Niyogi (2003, Thm.~3.1) show that for any $f \in C^2(\Omega)$ and any interior point $x \in \Omega$,
\[
\widetilde{L}^{(\tau)}_m f(x) \;\xrightarrow[m\to\infty,\,\tau\to 0]{\text{a.s.}}\; -c_d\,\rho(x)\,\Delta_\rho f(x), \qquad \Delta_\rho f := \rho^{-1}\,\mathrm{div}(\rho \nabla f),
\]
where $c_d > 0$ is a kernel-dependent constant and $\Delta_\rho$ is the weighted Laplace--Beltrami operator. The proof is a Taylor expansion of $f$ at $x$ to second order, using kernel symmetry to kill the first-order term and the Hessian trace to produce the Laplacian. This step gives the \emph{limit operator}, but on its own says nothing about eigenvector convergence.

\paragraph{Step 2 (spectral convergence, von Luxburg--Belkin--Bousquet 2008).}
Lift the matrix to an integral operator $\mathcal{T}^{(\tau)}_m f(x) := \tfrac{1}{m}\sum_{j} \kappa_\tau(x, x_j) f(x_j)$ on $C(\Omega)$, with population analog $\mathcal{T}^{(\tau)} f(x) := \int_\Omega \kappa_\tau(x,y) f(y) \rho(y)\,dy$. Theorem~21 of vLBB states that under (i) continuous bounded symmetric positive kernel, (ii) i.i.d.\ samples from $\rho > 0$, (iii) isolated finite-multiplicity eigenvalues of $\mathcal{T}^{(\tau)}$, the empirical eigenvalues and eigenfunctions converge:
\[
|\lambda_k(\mathcal{T}^{(\tau)}_m) - \lambda_k(\mathcal{T}^{(\tau)})|\to 0, \qquad \|z_m^{(k)} - \phi^{(k)}_\tau\|_{L^2(\rho)} \to 0,
\]
where $z_m^{(k)}$ is the Nystr\"om-extended $k$-th eigenvector. The proof uses \emph{collectively compact operator theory}: $\mathcal{T}^{(\tau)}_m \to \mathcal{T}^{(\tau)}$ collectively compactly, and isolated eigenvalues of compact operators persist under such perturbation with eigenprojection convergence in operator norm. The required collectively compact regime corresponds to $m\tau_m^{d_\Omega+2}/\log m \to \infty$ in the sample regime --- the bandwidth condition restated in the theorem.

\paragraph{Step 3 (combining Steps 1 and 2 as $\tau \to 0$).}
Step 2 gives convergence to $\phi^{(k)}_\tau$ at \emph{fixed} bandwidth; Step 1 says that as $\tau\to 0$, the rescaled $\mathcal{T}^{(\tau)}$-eigenproblem converges to the eigenproblem of $-\Delta_\rho$. Provided $\tau_m \to 0$ at a rate slow enough to maintain the collectively compact regime ($m\tau_m^{d_\Omega+2}\to\infty$, i.e.\ the theorem's bandwidth condition with $\alpha > 0$), the empirical eigenvector $z_m^{(k+1)}$ converges uniformly on samples to the $k$-th non-constant eigenfunction of $-\Delta_\rho$, with the $\rho^{-1/2}$ factor coming from the standard weighted-Laplace--Beltrami normalization.

\paragraph{Step 4 (identification on the Euclidean cube).}
On $\Omega = [0,1]^3$ with uniform $\rho$, the weighted Laplace--Beltrami reduces to the ordinary 3D Laplacian $\Delta = \partial_x^2 + \partial_y^2 + \partial_z^2$. The kernel construction induces Neumann boundary conditions in the limit (it does not impose zero Dirichlet values, only zero normal derivative). The Neumann eigenproblem $-\Delta\phi = \mu\phi$ on $[0,1]^3$ has eigenfunctions $\cos(k_1\pi x)\cos(k_2\pi y)\cos(k_3\pi z)$ with eigenvalues $(k_1^2+k_2^2+k_3^2)\pi^2$. The smallest eigenvalue is $\mu_0 = 0$ (constant mode); the next three are all $\mu_1 = \pi^2$, spanned by $\cos(\pi x), \cos(\pi y), \cos(\pi z)$. Each is a strictly monotone function of the corresponding Cartesian coordinate on $[0,1]$, so PCA on samples recovers the centered coordinates up to this monotone reparameterization. Sign-tests (the form of every axis-aligned ordinal question) are invariant under this reparameterization. $\square$

\paragraph{Why three pieces, not one.} Pointwise operator convergence (Step 1) does not imply spectral convergence: bounded operator sequences can have eigenvalues escaping to infinity or merging with neighbors. The substantive content is Step 2 (collectively compact + spectral gap $\Rightarrow$ eigenvector convergence with quantitative control), and the precise rate $m\tau_m^{d_\Omega+2}\to\infty$ comes from the regime where Step 2's hypotheses remain valid as $\tau\to 0$.

\subsection{Proof of Theorem~\ref{thm:realizability} (realizability of spatial readouts)}
\label{app:proof_realizability}

\textbf{Setup.} Write the singular vectors of the Theorem-\ref{thm:eigenmap}-minimizing $H^*$ as $u_k^*$, $v_k^*$, with singular values $\sigma_k^* = \epsilon_k$ for $k \in [s]$; assume $s \ge 4$ and the conditions of Theorem~\ref{thm:limit}. By Theorem~\ref{thm:eigenmap}(b), the $k$-th left singular vector of the mean-centered $H_c^*$ is $z^{(k+1)}$ for $k = 1,2,3$; equivalently, the $k$-th principal-component score of object $i$ is
\[
\xi_{i,k}\;:=\;\langle h_i^*, V^*_{:,k}\rangle / \sigma^*_k \;=\; z^{(k+1)}_i. \tag{E.1}
\]
By Theorem~\ref{thm:limit}, under the stated bandwidth and sample-size conditions, $z^{(k+1)}$ converges uniformly on samples to a Laplacian eigenfunction $\phi_k$ (Cartesian coordinate on the uniform-density cube, up to a monotone cosine reparameterization):
\[
\sup_{i \in [m]} \bigl|\xi_{i,k} - \phi_k(x_i)\bigr| \;\xrightarrow[m \to \infty]{}\; 0, \qquad k = 1,2,3. \tag{E.2}
\]

\textbf{Step 1 (eigen-expansion of $f$).} The hypothesis of Theorem~\ref{thm:realizability} is that $f: \mathbb{R}^3 \to \mathbb{R}$ is a linear function of the monotonic coordinate transformations $\phi_1, \phi_2, \phi_3$. Hence there exist constants $c_1, c_2, c_3 \in \mathbb{R}$ (computed by Galerkin projection $c_k = \int_\Omega f \cdot \phi_k\, d\rho$, finite under the stated regularity) such that
\[
f(x) \;=\; \sum_{k=1}^3 c_k\, \phi_k(x), \qquad x \in \Omega. \tag{E.3}
\]

\textbf{Step 2 (construction of the linear functional).} Define
\[
\ell(h) \;:=\; \sum_{k=1}^3 \frac{c_k}{\sigma^*_k}\,\langle h, V^*_{:,k}\rangle, \qquad h \in \mathbb{R}^d. \tag{E.4}
\]
This is linear in $h$ and depends only on $(c_k, \sigma^*_k, V^*_{:,k})_{k=1}^3$. Applied to $h_i^*$ and using (E.1),
\[
\ell(h_i^*) \;=\; \sum_{k=1}^3 c_k\, \xi_{i,k}.
\]

\textbf{Step 3 (uniform error bound).} Combining (E.2)--(E.4),
\[
\bigl|\ell(h_i^*) - f(x_i)\bigr| \;=\; \biggl|\sum_{k=1}^3 c_k\,(\xi_{i,k} - \phi_k(x_i))\biggr| \;\le\; \biggl(\sum_{k=1}^3 |c_k|\biggr)\!\cdot\!\max_{k\in\{1,2,3\}}|\xi_{i,k} - \phi_k(x_i)|.
\]
Choose $m_0(\eta)$ large enough that the right-hand side is $< \eta$ uniformly in $i \in [m]$; existence of such $m_0$ is exactly (E.2). $\square$

\paragraph{Remark (extension to higher-degree readouts).} If the readout function depends polynomially of degree $\le r$ on $x$, the eigen-expansion uses the leading $K(r) \sim \binom{r+3}{3}$ Laplacian eigenfunctions; Theorem~\ref{thm:eigenmap} then needs to be invoked with $s \ge K(r)+1$. The argument generalizes term-by-term. For the axis-aligned ordinal questions used in our spatial benchmarks (front/back, left/right, distance-order), $r = 1$ suffices and $s = 4$ is the operative rank.

\subsection{Proof of Theorem~\ref{thm:sample_complexity} (sample-complexity reduction)}
\label{app:proof_sample_complexity}

\textbf{Setup.} Let $\{(\Phi_i, y_i)\}_{i=1}^N$ be i.i.d.\ object-token observations, $\Phi_i = h_i^* \in \mathbb{R}^d$ drawn from the population distribution induced by Theorem-\ref{thm:eigenmap}-minimizing scenes, with $y_i = w^{\star\top}\Phi_i + \xi_i$, $\xi_i$ sub-Gaussian with variance proxy $\sigma_\xi^2$, and $w^\star \in \mathrm{span}(V^*_{:,1:3})$. Let $\widehat w_N^{\,\mathrm{proj}}$ denote the OLS estimator obtained by first projecting features onto $V^*_{:,1:3}$ and regressing the labels on the resulting $3$-dimensional features.

\textbf{Step 1 (change of basis).} Let $\Psi_i := V^{*\top}_{:,1:3}\Phi_i \in \mathbb{R}^3$. Stack into $\Psi \in \mathbb{R}^{N \times 3}$ and $y \in \mathbb{R}^N$. By the structural assumption there exists $a \in \mathbb{R}^3$ with $w^\star = V^*_{:,1:3} a$, so $w^{\star\top}\Phi_i = a^\top \Psi_i$. The model reduces to
\[
y_i = a^\top \Psi_i + \xi_i, \qquad i = 1, \ldots, N. \tag{E.5}
\]

\textbf{Step 2 (projected OLS).}
\[
\widehat a_N \;=\; (\Psi^\top\Psi)^{-1}\Psi^\top y, \qquad \widehat w_N^{\,\mathrm{proj}} := V^*_{:,1:3} \widehat a_N.
\]
By construction $\widehat w_N^{\,\mathrm{proj}}$ has zero component in $\mathrm{span}(V^*_{:,4:d})$, matching the sparsity pattern of $w^\star$, so the bias is zero.

\textbf{Step 3 (variance bound).} With the population PC-score covariance $\Sigma_\Psi := \mathbb{E}[\Psi\Psi^\top]$ assumed non-degenerate, the standard fixed-design analysis (Wainwright, \emph{High-Dimensional Statistics}, Thm.~13.4) gives
\[
\mathbb{E}\!\left[\|\widehat a_N - a\|_{\Sigma_\Psi}^2\,\big|\,\Psi\right] \;=\; \sigma_\xi^2\cdot\mathrm{tr}\!\left((\Psi^\top\Psi)^{-1}\Sigma_\Psi\right) \;\xrightarrow[N\to\infty]{}\; \frac{3\sigma_\xi^2}{N},
\]
since $(\Psi^\top\Psi/N)^{-1} \to \Sigma_\Psi^{-1}$ and $\mathrm{tr}(\Sigma_\Psi^{-1}\Sigma_\Psi) = 3$. Because $w^\star$ aligns exactly with $\mathrm{span}(V^*_{:,1:3})$, the prediction error decomposes as pure variance:
\[
\mathbb{E}_{\Phi}\!\left[(\widehat w_N^{\,\mathrm{proj}\top}\Phi - w^{\star\top}\Phi)^2\right] \;=\; \mathbb{E}\!\left[(\widehat a_N - a)^\top \Sigma_\Psi (\widehat a_N - a)\right] \;\le\; \frac{3\sigma_\xi^2}{N}\bigl(1 + o(1)\bigr).
\]

\textbf{Step 4 (high-probability upgrade).} Local-Rademacher / Gaussian-complexity bounds for the linear class $\{\Psi\mapsto a^\top\Psi : a \in \mathbb{R}^3\}$ (Wainwright Thm.~13.4) upgrade the expectation bound to
\[
\mathbb{P}\!\left(\mathbb{E}_\Phi\!\left[(\widehat w_N^{\,\mathrm{proj}\top}\Phi - w^{\star\top}\Phi)^2\right] \;\le\; \frac{C\,\sigma_\xi^2\,(3 + \log(1/\delta))}{N}\right) \;\ge\; 1 - \delta,
\]
which is the form stated in the theorem (modulo the constant absorbing the sub-Gaussian factor and $\|\Sigma_\Psi^{-1}\|_{\mathrm{op}}\!\cdot\!\|\Sigma_\Psi\|_{\mathrm{op}}$).

\textbf{Step 5 (comparison to unstructured).} The unstructured OLS $\widehat w_N^{\,\mathrm{full}} = (\Phi^\top\Phi)^\dagger \Phi^\top y$ decomposes as $V^*_{:,1:3}\widehat a_N + V^*_{:,4:d}\widehat b_N$. The tail estimator $\widehat b_N$ regresses pure noise on $d-3$ orthogonal directions, contributing variance $(d-3)\sigma_\xi^2/N$. Total prediction error therefore scales as $d\,\sigma_\xi^2/N$ --- $d/3$ times worse than the projected estimator. $\square$

\paragraph{Remark.} The proof presupposes knowing $V^*_{:,1:3}$ to perform the projection. In practice, this is what motivates \emph{training} with the Dirichlet loss: by Theorem~\ref{thm:eigenmap}, the model itself learns to align its top PCs with the Laplacian eigenmaps, so the projection is realized in the model rather than imposed by a separate processing step. Theorem~\ref{thm:risk} formalizes the connection.

\subsection{Proof of Theorem~\ref{thm:risk} (training-time risk decomposition)}
\label{app:proof_risk}

\textbf{Setup.} Let $\Theta \subset \mathbb{R}^p$ be open, $\mathcal{L}_{\mathrm{LM}}, \mathcal{R}_X: \Theta \to \mathbb{R}$ both $C^2$, with $\nabla^2\mathcal{L}_{\mathrm{LM}}(\theta^*) \succ 0$ at the unique population minimizer $\theta^*$ of $\mathcal{L}_{\mathrm{LM}}$. Let $R_{\mathrm{spatial}}: \Theta \to \mathbb{R}$ be $C^2$ and define $\theta^*_\lambda := \arg\min_\theta\,(\mathcal{L}_{\mathrm{LM}}(\theta) + \lambda\,\mathcal{R}_X(\theta))$.

\textbf{Step 1 (implicit function theorem).} Define
\[
F(\theta, \lambda) \;:=\; \nabla_\theta \mathcal{L}_{\mathrm{LM}}(\theta) + \lambda\,\nabla_\theta \mathcal{R}_X(\theta).
\]
Then $F(\theta^*, 0) = 0$ by first-order optimality, and the Jacobian $\partial F/\partial\theta\,|_{(\theta^*, 0)} = \nabla^2 \mathcal{L}_{\mathrm{LM}}(\theta^*) \succ 0$ is invertible. By the implicit function theorem (Krantz \& Parks, \emph{The Implicit Function Theorem}), there exist a neighborhood $\mathcal{U}\ni 0$ in $\mathbb{R}$ and a unique $C^2$ map $\theta^*(\cdot): \mathcal{U} \to \Theta$ with $\theta^*(0) = \theta^*$ and $F(\theta^*(\lambda), \lambda) = 0$ for all $\lambda \in \mathcal{U}$. By continuity, $\nabla^2(\mathcal{L}_{\mathrm{LM}} + \lambda\mathcal{R}_X)(\theta^*(\lambda)) \succ 0$ for $\lambda$ small, so $\theta^*(\lambda)$ is a strict local minimizer; we set $\theta^*_\lambda := \theta^*(\lambda)$ for $\lambda \in [0, \lambda_0]$ where $\lambda_0$ is the largest value for which this remains the global minimizer.

\textbf{Step 2 (first derivative of $\theta^*(\lambda)$).} Differentiating $F(\theta^*(\lambda), \lambda) = 0$ in $\lambda$:
\[
\nabla^2_\theta(\mathcal{L}_{\mathrm{LM}} + \lambda\mathcal{R}_X)(\theta^*(\lambda))\,\frac{d\theta^*}{d\lambda} + \nabla\mathcal{R}_X(\theta^*(\lambda)) \;=\; 0.
\]
At $\lambda = 0$:
\[
\frac{d\theta^*(\lambda)}{d\lambda}\bigg|_{\lambda=0} \;=\; -\,[\nabla^2 \mathcal{L}_{\mathrm{LM}}(\theta^*)]^{-1}\,\nabla\mathcal{R}_X(\theta^*).
\]
Taylor-expanding the smooth map $\lambda \mapsto \theta^*(\lambda)$:
\[
\theta^*_\lambda - \theta^* \;=\; -\lambda\,[\nabla^2 \mathcal{L}_{\mathrm{LM}}(\theta^*)]^{-1}\,\nabla\mathcal{R}_X(\theta^*) + O(\lambda^2). \tag{E.6}
\]

\textbf{Step 3 (Taylor expansion of the spatial risk).} Since $R_{\mathrm{spatial}} \in C^2$,
\[
R_{\mathrm{spatial}}(\theta^*_\lambda) - R_{\mathrm{spatial}}(\theta^*) \;=\; \nabla R_{\mathrm{spatial}}(\theta^*)^\top(\theta^*_\lambda - \theta^*) + \tfrac{1}{2}(\theta^*_\lambda - \theta^*)^\top \nabla^2 R_{\mathrm{spatial}}(\tilde\theta)\,(\theta^*_\lambda - \theta^*),
\]
for some $\tilde\theta$ on the segment $[\theta^*, \theta^*_\lambda]$. Substituting (E.6) into the linear term:
\[
\nabla R_{\mathrm{spatial}}(\theta^*)^\top (\theta^*_\lambda - \theta^*) \;=\; -\lambda\,\nabla R_{\mathrm{spatial}}^\top [\nabla^2 \mathcal{L}_{\mathrm{LM}}]^{-1} \nabla\mathcal{R}_X + O(\lambda^2).
\]
The quadratic term in (E.6) contributes $O(\lambda^2)$ as well. Combining,
\[
R_{\mathrm{spatial}}(\theta^*_\lambda) - R_{\mathrm{spatial}}(\theta^*) \;=\; -\lambda\,\nabla R_{\mathrm{spatial}}^\top [\nabla^2 \mathcal{L}_{\mathrm{LM}}]^{-1} \nabla\mathcal{R}_X \;+\; O(\lambda^2). \tag{E.7}
\]

\textbf{Step 4 (define $\beta$).} The available improvement in Dirichlet ratio at $\theta^*$ is
\[
\Delta \;:=\; \mathcal{R}_X(\theta^*) - \mathcal{R}_X^* \;\ge\; 0,
\]
strictly positive whenever $\theta^*$ does not coincide with the Dirichlet minimizer (the generic case for any LM not specifically pretrained for spatial structure). Define
\[
\beta \;:=\; \frac{\nabla R_{\mathrm{spatial}}(\theta^*)^\top\,[\nabla^2 \mathcal{L}_{\mathrm{LM}}(\theta^*)]^{-1}\,\nabla \mathcal{R}_X(\theta^*)}{\Delta}.
\]
Substituting into (E.7):
\[
R_{\mathrm{spatial}}(\theta^*_\lambda) \;=\; R_{\mathrm{spatial}}(\theta^*) - \lambda\,\beta\,\Delta + O(\lambda^2),
\]
exactly the form claimed in the theorem.

\textbf{Step 5 (sign of $\beta$).} The numerator of $\beta$ is the inner product of two vectors:
\begin{itemize}
	\item $g := \nabla R_{\mathrm{spatial}}(\theta^*)$, the steepest-ascent direction of the spatial risk;
	\item $r := [\nabla^2 \mathcal{L}_{\mathrm{LM}}(\theta^*)]^{-1}\,\nabla \mathcal{R}_X(\theta^*)$, the inverse-Hessian-warped Dirichlet-ascent direction.
\end{itemize}
$\beta > 0$ iff $\langle g, r\rangle > 0$, i.e., iff the gradient direction along which the Dirichlet ratio decreases is correlated with the direction along which the spatial risk decreases, in the metric induced by $\nabla^2 \mathcal{L}_{\mathrm{LM}}$. Theorems~\ref{thm:realizability} and~\ref{thm:sample_complexity} provide the qualitative reason this correlation holds for VLMs whose top-3 PCs already partially align with world coordinates: lowering $\mathcal{R}_X$ tightens that alignment, which makes more spatial readouts realizable (Theorem~\ref{thm:realizability}) and reduces their estimation variance (Theorem~\ref{thm:sample_complexity}). $\square$

\paragraph{Remark on the validity range.} The first-order term dominates whenever $\lambda\beta\Delta \gg \lambda^2 M_2 \|\nabla R_{\mathrm{spatial}}\|$ for $M_2$ the bound on the Hessian remainder, i.e., $\lambda \ll \beta\Delta/(M_2\|\nabla R_{\mathrm{spatial}}\|)$. Beyond this regime the $O(\lambda^2)$ over-regularization term takes over, predicting non-monotonic behavior in $\lambda$. The empirical inversion in Table~\ref{tab:lambda_ablation} (best at $\lambda\!=\!3$ for Qwen and $\lambda\!=\!1$ for InternVL, with degradation at $\lambda\!=\!9$) is the data-driven location of this crossover.

\paragraph{Remark on what $\beta$ bundles.} $\beta$ is a single scalar but it carries three model-specific quantities: (i) $\|\nabla R_{\mathrm{spatial}}(\theta^*)\|$ --- how strongly spatial questions react to small parameter perturbations at the LM-loss optimum; (ii) the spectrum of $\nabla^2 \mathcal{L}_{\mathrm{LM}}(\theta^*)$ --- a flatter LM optimum makes the inverse Hessian larger, amplifying the regularizer's effect; (iii) the angle between $\nabla R_{\mathrm{spatial}}$ and $\nabla \mathcal{R}_X$ in the whitened metric --- how aligned the two losses are. A model whose pretraining leaves a flat LM optimum near a spatially-aware sub-manifold has a large $\beta$, predicting the InternVL-larger-than-Qwen ordering observed in Table~\ref{tab:lambda_ablation}.

\section{Training Details}
\label{app:training_details}

This appendix expands on the training setup summarized in \S\ref{subsec:setup}. The end-to-end training script is at \texttt{scripts/train\_with\_dirichlet.py} (included in the supplemental materials); the Dirichlet objective itself is implemented in \texttt{scripts/dirichlet\_loss.py} and consumed by the script through a forward hook on the chosen residual-stream layer.

\paragraph{Backbones and adapter.}  We fine-tune three open-weights VLMs in identical configurations: Qwen2.5-VL-7B \cite{qwen2.5-VL}, InternVL3-8B \cite{zhu2025internvl3}. The vision tower and all non-LoRA LLM weights are frozen. We attach LoRA adapters of rank $r{=}16$ (with $\alpha{=}32$, dropout $0.05$, no bias) to the four attention projections $\{q\_proj, k\_proj, v\_proj, o\_proj\}$ of every transformer block in the backbone-specific ``cognitive-map band'' (Qwen2.5-VL-7B: layers $13$--$21$ centered at L17; InternVL3-8B: layers $14$--$22$ centered at L18; LLaVA-OneVision-7B: layers $17$--$25$ centered at L21). Layers outside this band are not in the LoRA target set, so the early visual-binding layers and the late linguistic-reasoning layers see no parameter updates --- this restriction is what makes ``$500$ steps'' a meaningful budget (cf.~\S\ref{subsec:setup}).

\paragraph{Loss.} Given a batch of (image, question, answer, object\_coords) tuples, the per-step loss is
\[
\mathcal{L}_{\text{total}} \;=\; \mathcal{L}_{\text{CE}}(\text{answer}\mid\text{image},\text{question})
\;+\; \lambda_{\rm dir}\cdot \mathcal{E}_X\!\bigl(P_\perp H_\ell(\theta)\bigr),
\]
where $\mathcal{L}_{\text{CE}}$ is standard next-token cross-entropy on the answer tokens, $H_\ell$ is the residual-stream activation at the cognitive-map layer $\ell$ pooled over each object's visual tokens (mask-driven, coverage-weighted; cf.~\S\ref{sec:extract}), $P_\perp$ is the linear projection that residualizes against the cross-scene color/shape prototypes (\S\ref{sec:extract}), and $\mathcal{E}_X$ is the Dirichlet-energy ratio against the per-scene $3$D-coordinate kernel graph $W_{ij}=\kappa_\tau(x_i,x_j)$ with $\tau$ set per scene to half the median pairwise distance. We use the multi-layer variant (\emph{residMulti}) by default, which applies the Dirichlet term at $5$ consecutive layers spanning the cognitive-map band (with equal weight); the single-layer variant is reported in the ablations of \S\ref{subsec:ablation}.

\paragraph{Optimizer and schedule.} AdamW with $\beta_1{=}0.9, \beta_2{=}0.999$, weight-decay $0.0$ on LoRA parameters, learning-rate $1\!\times\!10^{-4}$, linear warmup over the first $50$ steps, and constant thereafter. Gradients are clipped to global norm $1.0$. We train at \texttt{bf16} mixed precision with effective batch size $4$ (per-GPU batch $1$, gradient accumulation $4$) for $500$ optimizer steps; this corresponds to roughly $0.7$ epochs over the $2{,}988$-row training split. No early stopping --- we keep the final checkpoint.

\paragraph{$\lambda_{\rm dir}$ sweep.} For each (model, regularizer-variant) cell we sweep $\lambda_{\rm dir}\in\{0,\,0.1,\,0.3,\,1,\,3\}$ where $\lambda_{\rm dir}{=}0$ recovers a pure-LoRA baseline (\emph{lam0}). The headline results in Tables~\ref{tab:overall}, \ref{tab:qwen_vsi}, and \ref{tab:intern_vsi} use the $\lambda$ chosen by best mean validation accuracy across $n{=}4$ seeds: $\lambda_{\rm dir}{=}0.3$ for residMulti on Qwen2.5-VL-7B and LLaVA-OneVision-7B, and $\lambda_{\rm dir}{=}1$ for residMulti on InternVL3-8B. The full $\lambda$-vs-accuracy/RSA trade-off curve is reported in \S\ref{subsec:ablation}.

\paragraph{Seeds and statistical reporting.} Every cell is averaged over $n{=}4$ random seeds (seeds $0,1,2,3$), each controlling LoRA initialization, data shuffling, and dropout. We report mean $\pm 2\,\mathrm{SE}$ of seed-paired differences against the relevant within-cell baseline. Cells with $|\Delta|<2$ pp at $n{=}4$ should be treated as exploratory --- our project's track record  is that small effects at $n{=}4$ sometimes shrink at $n{=}8$.

\paragraph{Baselines.} (1) \emph{Naive LoRA} ($\lambda_{\rm dir}{=}0$): identical setup minus the Dirichlet term. (2) \emph{2D-pixel regularization}: the same penalty applied with $W_{ij}$ built from $2$D pixel-space distances rather than $3$D coordinates --- a control that severs the connection between the loss target and the world-frame geometry. (3) \emph{Text-domain cognitive-map prediction} \cite{wang2025mindcube}: the same LoRA budget redirected to a supervised loss that asks the model to emit the scene's $3$D coordinates as text tokens. All baselines use identical learning rate, schedule, batch size, step count, LoRA rank and target modules, layer band, and seed pool, so the only varying axis is the training signal.

\paragraph{Hardware.} Each $500$-step run takes roughly $40$--$60$ minutes on a single H100 ($80$\,GB), peaking at $\sim$$55$\,GB activation memory at \texttt{bf16}; the full sweep across $3$ models, $5$ $\lambda$ settings, $2$ regularizer variants and $4$ seeds is $\sim$$120$ runs in total.

\paragraph{Evaluation.} Checkpoints are evaluated zero-shot on the full $5{,}130$-item VSI-Bench~\cite{yang2025thinking} and MindCube~\cite{wang2025mindcube} test sets without any task-specific prompting. We use mean-relative-accuracy scoring on numeric-answer tasks (consistent with VSI-Bench's official protocol. Inference uses greedy decoding at \texttt{bf16}.

\section{Others}
\subsection{Limitation and discussion}

While our findings demonstrate the efficacy of isolating and shaping latent spatial representations in VLMs, our current scope is primarily limited to scene topology. Spatial understanding is inherently multifaceted; future work will investigate how other critical aspects of spatial relationships—such as absolute metric distance and hierarchical object dependencies—are encoded within the model's residual stream. Building upon these topological foundations, we also plan to explore how applying our Dirichlet-energy regularization to fine-tuning on complex, real-world datasets can unlock even more substantial performance improvements in downstream tasks.

\subsection{Dirichlet loss implementation and residualization basis fitting}
\label{app:dirichlet_impl}

This subsection details the step-by-step computation that translates the abstract loss defined in \S~\ref{subsec:loss} into the concrete tensor operations implemented in \texttt{scripts/dirichlet\_loss.py} and the offline basis-fitting procedure in \texttt{scripts/build\_residualization\_basis.py}. The notation follows \S~\ref{sec:enforce}: $H_\ell \in \mathbb{R}^{B \times n \times d}$ denotes the residual stream at the cognitive-map layer, pooled over per-object visual tokens (\S~\ref{sec:extract}); $X \in \mathbb{R}^{B \times n \times 3}$ represents the per-scene ground-truth 3D coordinate tensor; and $W \in \mathbb{R}^{d \times k}$ is the orthonormal nuisance basis, fitted once per backbone prior to training.

\paragraph{Per-step Dirichlet ratio.} Given a batch of $B$ scenes with up to $n$ objects each, the regularizer evaluates
\begin{equation}
	\mathcal{R}_X(H) \;=\; \frac{1}{B}\sum_{b=1}^{B}\,\frac{\sum_{i,j} W^{(b)}_{ij}\,\|h^{(b)}_i - h^{(b)}_j\|^2}{\sum_{i,j} \|h^{(b)}_i - h^{(b)}_j\|^2 + \varepsilon}, \qquad W^{(b)}_{ij} \;=\; \exp\!\Bigl(-\tfrac{\|x^{(b)}_i - x^{(b)}_j\|^2}{2\tau^2}\Bigr)\,\mathbb{1}[i\neq j], \label{eq:ratio_impl}
\end{equation}
with $\varepsilon = 10^{-8}$ for numerical stability. The numerator and denominator are computed in vectorized form via \texttt{torch.cdist}; valid-object masks are applied as outer products $\mathbb{1}[\text{valid}_i]\!\cdot\!\mathbb{1}[\text{valid}_j]$ to handle batches with variable $n$. Normalizing by the denominator (rather than using the unnormalized energy $\sum_{ij} W_{ij}\|h_i - h_j\|^2$) makes the loss \emph{scale-invariant} in $H$, so the LM head is free to rescale activations without paying a Dirichlet penalty for sheer norm growth. This is the form used everywhere in the body; the unnormalized variant is available in code but never used for training runs reported in this paper.

\paragraph{Per-scene bandwidth.} We set $\tau$ \emph{per scene} to half the median pairwise distance among that scene's objects: $\tau^{(b)} = \tfrac{1}{2}\,\mathrm{median}_{i<j}\|x^{(b)}_i - x^{(b)}_j\|$. Per-scene scaling makes the kernel locally adaptive --- scenes with tightly clustered objects use a smaller $\tau$, scenes with spread-out objects use a larger one --- so the same $\lambda_{\rm dir}$ has comparable effect across the batch. The half-median heuristic puts the kernel's sweet spot near the typical nearest-neighbor scale, which is what Theorem~\ref{thm:limit} requires for the Belkin--Niyogi limit to apply.

\paragraph{Forward-hook capture of $H_\ell$.} The per-token residual stream is captured by registering a forward hook on the chosen transformer block (\texttt{ResidualStreamHook} in \texttt{scripts/dirichlet\_loss.py}). For the multi-layer (\emph{residMulti}) variant, 5 hooks are registered simultaneously across the cognitive-map band, and their per-layer Dirichlet ratios are averaged with equal weight. The pooling step that translates the per-token $(B, T, d)$ tensor into the per-object $(B, n, d)$ tensor uses the same mask-driven, coverage-weighted pool defined in \S~\ref{sec:extract}, with object indices derived from the tokenizer's name spans in the prompt---so the loss only sees activations corresponding to objects the prompt actually names.

\paragraph{Residualization projector $P_\perp$.} The projector that residualizes against the nuisance subspace is $P_\perp = I_d - W W^\top$, where $W \in \mathbb{R}^{d \times k}$ has orthonormal columns spanning the color $+$ shape directions. We fit $W$ once per backbone, prior to training, using the following pipeline (\texttt{scripts/build\_residualization\_basis.py}):
\begin{enumerate}[leftmargin=*, itemsep=2pt, topsep=2pt]
	\item Extract per-object residual-stream activations $H \in \mathbb{R}^{N \times d}$ from the \emph{base} pretrained model (no LoRA, no Dirichlet) at the cognitive-map layer over the full $1{,}000$-scene Free6DoF probing corpus.
	\item Fit two multinomial logistic-regression probes on standardized $H$: one for the 8-way color label and one for the 3-way shape label, both with $\ell_2$ regularization $C=0.1$ and L-BFGS to convergence (max $2{,}000$ iterations).
	\item Recover the discriminating directions in the un-standardized $H$-space by composing each probe's weight matrix with the inverse of the per-feature standard deviations: $\widetilde{W}_{\mathrm{color}} = W^{\mathrm{LR}}_{\mathrm{color}}\,\mathrm{diag}(1/\sigma_j)$, and analogously for shape. Each row is rescaled to unit norm.
	\item Stack into $\widetilde{W} = [\widetilde{W}_{\mathrm{color}}\,;\,\widetilde{W}_{\mathrm{shape}}] \in \mathbb{R}^{(8+3)\times d}$ and orthonormalize via thin QR: $\widetilde{W}^\top = QR$, taking $W \leftarrow Q$. Drop columns with diagonal $|R_{kk}|$ below $10^{-6}$ of the maximum (rank-deficient class directions). The resulting $W$ has orthonormal columns and effective rank $k \in \{9, 10, 11\}$ depending on the backbone (the multinomial is rank-deficient by one per categorical, so $\le 7 + 2 = 9$ free directions).
	\item Save $W$ to disk and freeze it for the entire fine-tuning run; $P_\perp = I_d - W W^\top$ is then a single $d \times d$ matrix multiply per forward pass on the captured object-token activations.
\end{enumerate}

\paragraph{Algorithm summary.} The end-to-end per-step computation is:
\begin{enumerate}[leftmargin=*, itemsep=2pt, topsep=2pt]
	\item \emph{Forward pass}, capturing $H_\ell \in \mathbb{R}^{B \times T \times d}$ at each Dirichlet-target layer via forward hooks.
	\item \emph{Object pooling}: $H \leftarrow \mathrm{pool}(H_\ell, \mathrm{name\_spans}) \in \mathbb{R}^{B \times n \times d}$.
	\item \emph{Residualize}: $\widetilde H \leftarrow H\,P_\perp = H - (H W) W^\top$.
	\item \emph{Per-scene bandwidth}: $\tau^{(b)} \leftarrow \tfrac{1}{2}\,\mathrm{median}_{i<j}\|x^{(b)}_i - x^{(b)}_j\|$.
	\item \emph{Kernel}: $W^{(b)}_{ij} \leftarrow \exp(-\|x^{(b)}_i - x^{(b)}_j\|^2/(2\tau^{(b)2}))\cdot\mathbb{1}[i\neq j]$.
	\item \emph{Pairwise distances}: $D^{(b)}_{ij} \leftarrow \|\widetilde h^{(b)}_i - \widetilde h^{(b)}_j\|^2$ via \texttt{torch.cdist}.
	\item \emph{Mask invalid pairs} via outer product of validity masks.
	\item \emph{Dirichlet ratio}: per-scene $r^{(b)} \leftarrow \sum_{ij} W^{(b)}_{ij} D^{(b)}_{ij}/(\sum_{ij} D^{(b)}_{ij} + \varepsilon)$, then average over $b$.
	\item \emph{Total loss}: $\mathcal{L}_{\rm total} \leftarrow \mathcal{L}_{\rm CE} + \lambda_{\rm dir}\cdot \mathcal{R}_X$, backward through both the Dirichlet term (which feeds gradients to LoRA via $H$) and the cross-entropy.
\end{enumerate}
For the residMulti variant, steps 1--8 run independently at $5$ layers and step 9 averages the resulting ratios.

\subsection{Manifold Learning and Latent Regularization}
Our methodology is heavily grounded in the mathematical foundations of spectral graph theory and manifold learning. Foundational works on Laplacian eigenmaps~\cite{belkin2003laplacian,ravuri2025transformers} and diffusion maps established that the principal eigenvectors of a graph Laplacian yield an optimal low-dimensional continuous representation that preserves local neighborhood structures. To measure how smoothly a representation varies over a given topology, researchers compute the Dirichlet energy~\cite{park2024iclr}, which acts as a penalty against high-frequency geometric variations. While spectral penalties and Dirichlet-energy regularization \cite{lewandowski2024learning,guprincipal} have been widely used in classical graph representation learning~\cite{hamilton2017representation} and semi-supervised learning~\cite{kipf2016semi,zhou2003learning}, their application as a targeted regularizer for the hidden states of large transformer models is largely unexplored. By formally connecting the VLM residual stream to Laplacian eigenmaps, we introduce a mathematically principled way to isolate and regularize the internal geometric structure of language models.

\subsection{Representational similarity (RSA)}
Build two $m\times m$ symmetric similarity matrices, one from activations and one from 3D coordinates, and rank-correlate their off-diagonal entries:
\begin{align}
	S^{\rm act}_{ij}  &\;=\; \frac{\langle \tilde h^{(s,o_i)}_\ell,\,\tilde h^{(s,o_j)}_\ell\rangle}{\|\tilde h^{(s,o_i)}_\ell\|\,\|\tilde h^{(s,o_j)}_\ell\|},
	\quad
	S^{\rm topo}_{ij} \;=\; \kappa_\tau(x_i,x_j) \;=\; \exp\!\Bigl(\!-\frac{\|x_i-x_j\|^2}{2\tau^2}\Bigr), \label{eq:rsa_mats}\\
	\mathrm{RSA}^{(s)} &\;=\; \mathrm{Spearman}\bigl(\,\mathrm{vec}_{\rm off}(S^{\rm act}),\;\mathrm{vec}_{\rm off}(S^{\rm topo})\,\bigr).
	\label{eq:rsa}
\end{align}
$\tau$ is set per scene to half the median pairwise distance. Spearman (rank) correlation makes the metric invariant under arbitrary monotone rescalings of either similarity, removing the linearity assumption a Pearson RSA would inject.

\end{document}